\documentclass[twoside]{article}

\usepackage[accepted]{aistats2024}
\usepackage{my_style}
\usepackage[round]{natbib}

\usepackage{enumitem}

\usepackage{amsthm}
\usepackage{amssymb}
\usepackage{mathtools}
\usepackage{bm}
\usepackage{bbm}
\usepackage{graphicx}
\usepackage{float}
\usepackage{xcolor}
\usepackage{multicol}
\usepackage{outlines}
\usepackage{thmtools,thm-restate}

\usepackage[textwidth=1.8cm,textsize=scriptsize]{todonotes}
\usepackage{amsbsy}
\usepackage{comment}
\usepackage[algo2e,boxed,vlined]{algorithm2e}
\SetAlgorithmName{Algorithm}{Algorithm}{Algorithm}

\SetCommentSty{mycommfont}

\usepackage{diagbox}
\usepackage{multirow}
\usepackage{array}
\usepackage{makecell}
\usepackage{mathrsfs}

\newtheorem{theorem}{Theorem}
\newtheorem{lemma}[theorem]{Lemma}

\newtheorem{definition}{Definition}

\newtheorem{remark}{Remark}

\setboolean{long}{true}

\begin{document}

%

%

\twocolumn[

\aistatstitle{
Multitask Online Learning: Listen to the Neighborhood Buzz
}

\aistatsauthor{Juliette Achddou \And Nicolò Cesa-Bianchi \And Pierre Laforgue}

\aistatsaddress{Università degli Studi di Milano \\ Italy \And Università degli Studi di Milano \\ and Politecnico di Milano, Italy \And Università degli Studi di Milano \\ Italy} ]

\begin{abstract}
We study multitask online learning in a setting where agents can only exchange information with their neighbors on a given arbitrary communication network.
We introduce \mtcool, a decentralized algorithm for this setting whose regret depends on the interplay between the task similarities and the network structure. Our analysis shows that the regret of \mtcool is never worse (up to constants) than the bound obtained when agents do not share information.
On the other hand, our bounds significantly improve when neighboring agents operate on similar tasks.
In addition, we prove that our algorithm can be made differentially private with a negligible impact on the regret.
%
Finally, we provide experimental support for our theory.
\end{abstract}

\section{\MakeUppercase{Introduction}}\label{sec:motiv}
Many real-world applications, including recommendation, personalized medicine, or environmental monitoring, require learning a personalized service offered to multiple clients.
These problems are typically studied using multitask learning, or personalized federated learning when privacy is of concern. The key idea behind these techniques is that sharing information among similar clients may help learn faster.
In this work we study multitask learning in an online convex optimization setting where multiple agents share information across a communication network to minimize their local regrets.
We focus on decentralized algorithms, that operate without a central coordinating entity and only have a local knowledge of the communication network.
Motivated by scenarios in which long-range communication is costly (e.g., in sensor networks) or slow (e.g., in advertising/financial networks, where data arrive at very high rates), we assume agents can only communicate with their neighbors in the network.
Our regret analysis applies to the general multitask setting, where each agent is solving a potentially different online learning problem.
In our decentralized environment, we do not require agents to work in a synchronized fashion. Rather, agents predict according to some unknown sequence of agent activations.
We consider both deterministic (i.e., oblivious adversarial) or stochastic activation sequences.

It is well known that the optimal regret in single-agent single-task online convex optimization is $\scO(\sqrt{T})$, achieved, for example, by the \FTRL\ algorithm \citep{orabona2019modern}.
In the multi-agent setting with $N$ agents, one can trivially achieve regret $\scO\big(\sqrt{NT}\big)$ simply by running $N$ independent instances of \FTRL\ (no communication).
\citet{cesa2020cooperative} consider a multi-agent single-task setting where each active agent sends the current loss to their neighbors.
In this setup, they show a regret bound in $\scO\big(\sqrt{\alpha(G)\,T}\big)$, where $\alpha(G)\le N$ is the independence number of the communication graph $G$ (unknown to the agents).
They also show that when agents know $G$ and active agents have access to the predictions of their neighbors, then the regret bound becomes $\scO\big(\sqrt{\gamma(G)\, T}\big)$, where $\gamma(G) \le \alpha(G)$ is the domination number of $G$.
The multi-agent multitask setting in online convex optimization was studied by \citet{cesa2022multitask} for the case  when $G$ is a clique.
They prove a regret bound of order $\sqrt{1 + \sigma(N - 1)}\sqrt{T}$, where $\sigma$ is the variance of the set of best local models for the tasks. This bound is achieved by a multitask variant of \FTRL\ based on sharing the loss gradients.

In this work we consider the multitask setting when $G$ is arbitrary, which---to the best of our knowledge---was never investigated in the context of online convex optimization.
We introduce \Coolcn, a new decentralized variant of \FTRL\ where agents can only communicate with their neighbors.
We show $\mathcal{O}\big(\sqrt{T}\big)$ regret bounds, for adversarial and stochastic activations, in which the scaling $\sqrt{1 + \sigma(N - 1)}$ for the clique is replaced by $\sqrt{1 + \sigma_j(N_j - 1)}$ summed over agents $j$, where $\sigma_j$ is the task variance in the neighborhood of $j$ of size $N_j$ (\Cref{thm:any-G}).
We also recover the previously known bounds for single-task and multitask settings as special cases of our result.
\Cref{fig:expe_1} shows that, when the task similarity is in the right range, \Coolcn outperforms both multi-agent \texttt{FTRL} without communication among agents, and multi-agent single-task \texttt{FTRL}.
See \Cref{sec:expe} for more details and experiments.

\begin{figure}[t]
\centering
\includegraphics[width=0.8\columnwidth]{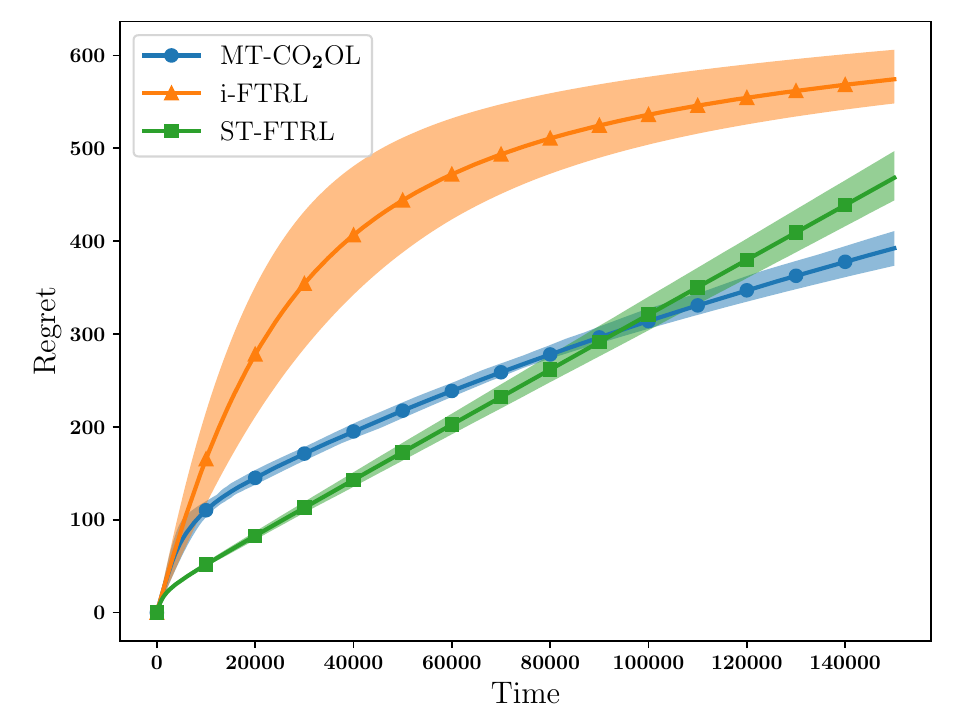}
\caption{
Multitask regret over time of \Coolcn on a random communication graph with stochastic activations against two baselines: \texttt{i-FTRL} ($N$ independent instances of \FTRL) and \texttt{ST-FTRL} (the multi-agent single-task algorithm of \citet{cesa2020cooperative}).
}
\label{fig:expe_1}
\end{figure}


We also prove a lower bound (\Cref{thm:lower-bound}) showing that our regret upper bounds are tight in some important special cases, such as regular communication graphs. Finally, we show that \Coolcn can be made differentially private, and prove that
privacy only degrades the multitask regret by a term polylogarithmic in $T$ (see \Cref{th:reg-DP_anyG}). This allows us to identify a privacy threshold above which sharing information no longer benefits the agents.

\subsection{Related works}
We review related works by focusing on multi-agent online learning with adversarial (nonstochastic) losses. In this area, we distinguish three main threads of research: single-task (or cooperative) learning, multitask (or heterogeneous) learning, and distributed optimization. We also distinguish between synchronous (all agents are active in each round) and asynchronous (one agent is active in each round) activation models. To the best of our knowledge, the study of cooperative online learning was initiated by \citet{awerbuch2008competitive}, who studied a cooperative-competitive synchronous bandit model in which agents are partitioned in unknown clusters and agents in the same cluster receive the same losses (i.e., solve a single-task problem) whereas the losses of agents in distinct clusters may be different.

\textbf{Single-task.}
The synchronous bandit model was also studied by \citet{cesa2019delay} and \citet{bar2019individual} in a network of agents where communication delays (based on the shortest-path distance) are taken into account. Their analysis was extended to linear bandits by \citet{ito2020delay}, and to linear semi-bandits by \citet{della2021efficient}.
\citet{cesa2020cooperative} investigate the asynchronous online convex optimization setting without delays and where agents can only talk to their neighbors. Their analysis is extended to neighborhoods of higher order by \citet{van2022distributed}, who also take communication delays into account. \citet{hsieh2022multi} and \cite{jiang2021asynchronous} consider the same setting, but with a more abstract model of delays, not necessarily induced by shortest-path distance on a graph.

\textbf{Multitask.} One of the earliest contributions in multitask online learning is by \citet{cavallanti2010linear}, who introduce an asynchronous multitask version of the Perceptron algorithm for binary classification.
\citet{murugesan2016} extend this idea to a model where task similarity is also learned.
Another extension is considered by \citet{cesa2022multitask}, who introduce a multitask version of Online Mirror Descent with arbitrary regularizers.
See also \citet{saha2011,zhang2018,li2019} for additional multitask online algorithms without performance bounds.
Recently, \citet{herbster2021gang} investigated an asynchronous bandit model on a social network in which the network's partition induced by the labeling assigning the best local action to each node has a small resistance-weighted cutsize.
Finally, note the work by \citet{sinha2023playing}, where the authors treat tasks as constraint generators (instead of loss generators), and achieve sublinear regret for online convex optimization.

\textbf{Distributed online optimization.}
\citet{yan2012distributed} introduced a synchronous multitask setting where the comparator is the best global prediction for all tasks and the system's performance is measured according to $\max_{j\in [N]} \sum_{i=1}^N \sum_{t=1}^T \ell_t^i\big(x_t^j)$, where $\ell_t^i(x_t^j)$ is the loss at time $t$ for the task of agent $i$ evaluated on the prediction at time $t$ of agent $j$.
Crucially, at time $t$ each agent $j$ only observes $\nabla\ell_t^j(x_t^j)$, and communication is used to gather information about other tasks.
\citet{hosseini2013online} extended the analysis from strongly convex losses to general convex losses.
\citet{yi2023doubly} consider the bandit case, in which each agent observes the loss incurred rather than the gradient.
Note that the synchronous bandit model by \citet{cesa2019delay} is a special case of this setting (when $\ell_t^1 = \ldots = \ell_t^N$ for all $t$).
However, their regret bound scales with the independence number of the communication graph, as opposed to spectral quantities such as the inverse of the spectral gap in distributed online optimization.

\section{\MakeUppercase{Problem Setting}}
\label{sec:model}

In this section, we introduce the multitask online learning setting, and recall important existing results.

\paragraph{Setting.}
Consider $N$ learning agents located at the nodes of a communication network described by an undirected graph $G = (V,E)$, where $V = [N] \coloneqq \{1,\ldots,N\}$.
Let $\mathcal{N}_i = \{i\} \cup \{j \in [N] \,:\, (i, j) \in E\}$ be the neighbors of agent $i$ in $G$ (including $i$ itself), and $N_i = |\mathcal{N}_i|$.
We denote $N_\textnormal{min} = \min_{i \in V} N_i$, and $N_\textnormal{max} = \max_{i \in V} N_i$.
We assume that agents ignore the graph, but know the identity of their neighbors.
Let $\mathcal{X} \subset \mathbb{R}^d$ be the agents common decision space, and $\ell_1,\ell_2,\ldots$ a sequence of convex losses from $\mathcal{X}$ to $\mathbb{R}$, secretly chosen by an oblivious adversary.
The learning process is as follows.
For $t=1, 2, \ldots$
\begin{enumerate}[topsep=0pt,parsep=0pt,itemsep=1pt]
\item some agent $i_t \in [N]$ is activated
\item \label{item:fetch} $i_t$ may \emph{fetch} information from its neighbors $j \in \mathcal{N}_{i_t}$
\item \label{item:pred} $i_t$ predicts $x_t \in \mathcal{X}$
\item \label{item:loss} $i_t$ incurs the loss $\ell_t(x_t)$ and observes $g_t \in \partial\ell_t(x_t)$
\item \label{item:send} $i_t$ may \emph{send} information to its neighbors $j \in \mathcal{N}_{i_t}$\,.
\end{enumerate}

Communication is limited to steps~\ref{item:fetch} and~\ref{item:send}, and only along edges incident on the active agent, hereby enforcing a decentralized learning process.
As in distributed optimization algorithms \citep{hosseini2013online}, we use step~\ref{item:fetch} to fetch model-related information, and step~\ref{item:send} to send gradient-related information, see \Cref{sec:comm-graph,sec:DP}.
%
As shown in our analysis, both communication steps are key to obtain strong regret guarantees.

We measure performance through the \emph{multitask regret}, defined as the sum of the agents' individual regrets.
The individual regret measures the performance of a single agent against the best decision in hindsight for its \emph{own personal task}, i.e., for the sequence of losses associated to the time steps it is active.
Namely, agent $i$ aims at minimizing for any $u^{(i)} \in \mathcal{X}$ its local regret $\sum_t \big(\ell_t(x_t) - \ell_t(u^{(i)})\big)\,\mathbb{I}\{i_t = i\}$, and the global objective is thus to control for any $u^{(1)}, \ldots, u^{(N)} \in \mathcal{X}$ the multitask regret $\sum_{i=1}^N \sum_{t=1}^T \big(\ell_t(x_t) - \ell_t(u^{(i)})\big)\,\mathbb{I}\{i_t = i\}$ .
Equivalently, this amounts to minimizing
\begin{equation}\label{eq:mt-regret}
\vspace{-0.15cm}
R_T(U) = \sum_{t=1}^T \ell_t(x_t) - \ell_t(U_{i_t:})   
\end{equation}
for any  horizon $T$ and multitask comparator $U \in \mathcal{U}$, where $\mathcal{U} \coloneqq \big\{U \in \mathbb{R}^{N \times d} \colon U_{i:} \in \mathcal{X} \text{ for all }i\big\}$.
Note that for simplicity in the rest of the paper we set $\mathcal{X} = \{x \in \mathbb{R}^d \colon \|x\|_2 \le 1\}$, and assume that all loss functions are $1$-Lipschitz, i.e., we have $\|g_t\|_2 \le 1$ for all $t$.
Our analysis can be readily extended to any ball of generic radius $D$ and $L$-Lipschitz losses, up to scaling each regret bound by $DL$.
As pointed out in the introduction, a naive approach to minimizing \eqref{eq:mt-regret}
consists in running $N$ independent instances of \texttt{FTRL},
%
without making use of 
steps \ref{item:fetch} and \ref{item:send}.
By Jensen's inequality, such strategy satisfies for any $U \in \mathcal{U}$
\begin{equation}\label{eq:regret_indep}
R_T(U) = \sum_{i=1}^N 2\sqrt{T_i} \le 2\sqrt{NT}\,,
\end{equation}
where $T_i = \sum_{t=1}^T \mathbb{I}\{i_t=i\}$.
The purpose of this work is to introduce and analyze a new algorithm whose regret improves on \eqref{eq:regret_indep} in terms of $N$.
Our bounds should depend on the interplay between the task similarity 
and the structure of the communication graph.


\begin{remark}[Comparison to \citet{cesa2020cooperative}]
Although \citet{cesa2020cooperative} also consider arbitrary communication graphs, their setting significantly differs from ours.
Recall that they work in a single-task setting (where $u^{(1)} = u^{(2)} = \ldots = u^{(N)}$). Hence, their proof techniques are significantly different from ours.
%
%
Moreover, while we only rely on gradient feedback, where the gradient is evaluated at the current prediction $x_t$, in their setting the learner is free to compute gradients at any point in the decision space. On the other hand, their communication model is restricted to sharing gradients while we can also fetch predictions.
%
\end{remark}

\paragraph{Preliminaries.}
%
%

Our algorithm  builds upon \MTFTRL \citep{cesa2022multitask}, designed for the clique case. 
We recall important facts about this algorithm.
\MTFTRL combines an instance of \texttt{FTRL} with a carefully chosen Mahalanobis regularizer and uses an adaptive learning rate based on \texttt{Hedge}---see \Cref{alg:mt-ftrl}, where $G_t = e_{i_t} g_t^\top \in \mathbb{R}^{N \times d}$ denotes the gradient matrix full of $0$ except for row $i_t$ which contains $g_t$.
The next result bounds the regret suffered by \MTFTRL.

\begin{theorem}[\protect{\citet[Theorem 9]{cesa2022multitask}}]\label{thm:mt-ftrl}
Let $G$ be a clique. The regret of \textnormal{\texttt{MT-FTRL}} with $\beta_{t-1}=\sqrt{t}$ satisfies for all $U \in \mathcal{U}$
\begin{equation}\label{eq:reg-mt-ftrl}
R_T(U) \tildeO \sqrt{1 + \sigma^2(N-1)}\sqrt{T}\,,
\end{equation}
where $\sigma^2 = \sigma^2(U) = \frac{1}{N-1} \sum_{i=1}^N \big\|U_{i:} - \frac{1}{N} \sum_{j=1}^N U_{j:}\big\|_2^2$ is the comparator variance.\footnote{
We use $g \tildeO f$ to denote $g = \widetilde{\mathcal{O}}(f)$, where $\widetilde{\mathcal{O}}$ hides logarithmic factors in $N$.
}
\end{theorem}

\SetKwInput{Req}{Requires}
\SetKwInput{Init}{Init}

\begin{algorithm}[t]
\caption{~\texttt{MT-FTRL} (on linear losses)}\label{alg:mt-ftrl}

\Req{Number of agents $N$, learning rates $\beta_{t-1}$}
\vspace{0.1cm}

\Init{$A = (1+N) I_N - \bm{1}_N\bm{1}_N^\top$,~~~$\Xi = \{1/N,\,2/N, \ldots, 1\}$,\\
\hspace{0.93cm}$p_1^{(\xi)} = \frac{1}{N}$~~for all $\xi \in \Xi $}
\vspace{0.05cm}

\For{$t = 1, 2, \ldots$}{
\vspace{0.05cm}

\For{$\xi \in \Xi$}{\vspace{0.05cm}

\tcp{\small \textcolor{blue}{Set learning rate assuming $\sigma^2 = \xi$}}
\vspace{0.05cm}

$\eta^{(\xi)}_{t-1} = \frac{N}{\beta_{t-1}}\sqrt{1+\xi(N-1)}$\vspace{0.15cm}

\tcp{\small \textcolor{blue}{FTRL with Mahalanobis regularizer}}
$\displaystyle X^{(\xi)}_t = \argmin_{\substack{X \in \mathcal{U}\\\sigma^2(X) \le \xi}} ~ \eta^{(\xi)}_{t-1} \left\langle \sum_{s \le t-1}G_s, X\right\rangle + \frac{1}{2} \|X\|_A^2$ \label{li:FTRL}

}
\vspace{0.15cm}

\tcp{Average the experts predictions}

Predict $X_t = \sum_{\xi \in\,\Xi} ~ p_t^{(\xi)}\,X_t^{(\xi)}$ \label{li:average_experts}
\vspace{0.15cm} 

Incur linearized loss $\langle g_t, [X_t]_{i_t:}\rangle$ and observe $g_t$
\vspace{0.15cm}

\tcp{Update $p_t$ based on the experts losses}
\vspace{0.05cm}

\For{$\xi \in \Xi$}{

\hspace{-0.15cm}$\displaystyle p_{t+1}^{(\xi)} = \frac{\exp\left(-\frac{\sqrt{\ln N}}{\beta_t}\sum_{s=1}^t\big\langle g_s, [X_s^{(\xi)}]_{i_s:}\big\rangle\right)}{\sum_k\,\exp\left(-\frac{\sqrt{\ln N}}{\beta_t} \sum_{s=1}^t \big\langle g_s, [X_s^{(k)}]_{i_s:}\big\rangle\right)}$ \label{li:compute_pt}
}
}
\end{algorithm}

Note that \eqref{eq:reg-mt-ftrl} is at most of the same order as the naive bound \eqref{eq:regret_indep}, as $\sigma^2(U) \le 8$ for all $U$. 
On the other hand, \eqref{eq:reg-mt-ftrl} gets better as $\sigma^2$ decreases, i.e., as the tasks get similar, and recovers the single-agent bound $\sqrt{T}$ when $\sigma^2 = 0$, i.e., when all tasks are identical.
Finally, we stress that \texttt{MT-FTRL} adapts to the true comparator variance $\sigma^2(U)$  without any prior knowledge about it.

The above approach crucially relies on the fact that the communication graph $G$ is a clique.
In the next section, we show how to use \MTFTRL as a building block to devise an algorithm that can operate on any communication graph.
Some of our bounds depend on notable parameters of the graph $G$, that we recall now.

\begin{definition}
Let $G = (V, E)$. A subset $S \subset V$ is\vspace{-0.1cm}
\begin{itemize}[topsep=0pt,itemsep=1pt,wide]
%
%
\item \emph{$k$-times independent} in $G$ if the shortest path between any two vertices in $S$ is of length $k+1$. The cardinality of the largest $k$-times independent set of $G$ is denoted $\alpha_k(G)$. For $k=1$, we use the term \emph{independent set} and adopt the notation $\alpha(G)$.
\item \emph{dominating} in $G$ if any vertex in $G \setminus S$ has a neighbor in $S$. The cardinality of the smallest dominating set of $G$ is denoted $\gamma(G)$ and called \emph{domination number of $G$}.\looseness-1
\end{itemize}

It is well known that: $\alpha_2(G) \le \gamma(G) \le \alpha(G) \le |V|$.
\end{definition}

\section{\MakeUppercase{Algorithm and Analysis}}
\label{sec:comm-graph}

\begin{algorithm}[t]
\caption{~\Coolcn}\label{alg:global_alg}
\Req{Base algorithm \Aclique, weights $w_{ij}$}\vspace{0.07cm}

\For{$t=1, 2, \ldots$}{\vspace{0.1cm}

Active agent $i_t$\vspace{0.07cm}
    
\quad fetches $\big[Y_t^{(j)}\big]_{i_t:}$ from each $j \in \mathcal{N}_{i_t}$
    
\quad predicts $x_t = \sum_{j \in \mathcal{N}_{i_t}} w_{i_tj}\big[Y_t^{(j)}\big]_{i_t:}$\vspace{0.1cm}
    
\quad pays $\ell_t(x_t)$ and observes $g_t \in \partial\ell_t(x_t)$\vspace{0.14cm}
    
\quad sends $\big(i_t, w_{i_tj}\,g_t\big)$ to each $j \in \mathcal{N}_{i_t}$\vspace{0.1cm}

\For{$j \in \mathcal{N}_{i_t}$}{\vspace{0.1cm}

Agent $j$ feeds the linear loss $\langle w_{i_t j}\,g_t, \cdot\,\rangle$ to their local instance of \Aclique\ and obtains $Y^{(j)}_{t+1}$
}
}
\end{algorithm}

\begin{figure*}[t]
\centering
\hspace{1.3cm}\includegraphics[height=3.5cm]{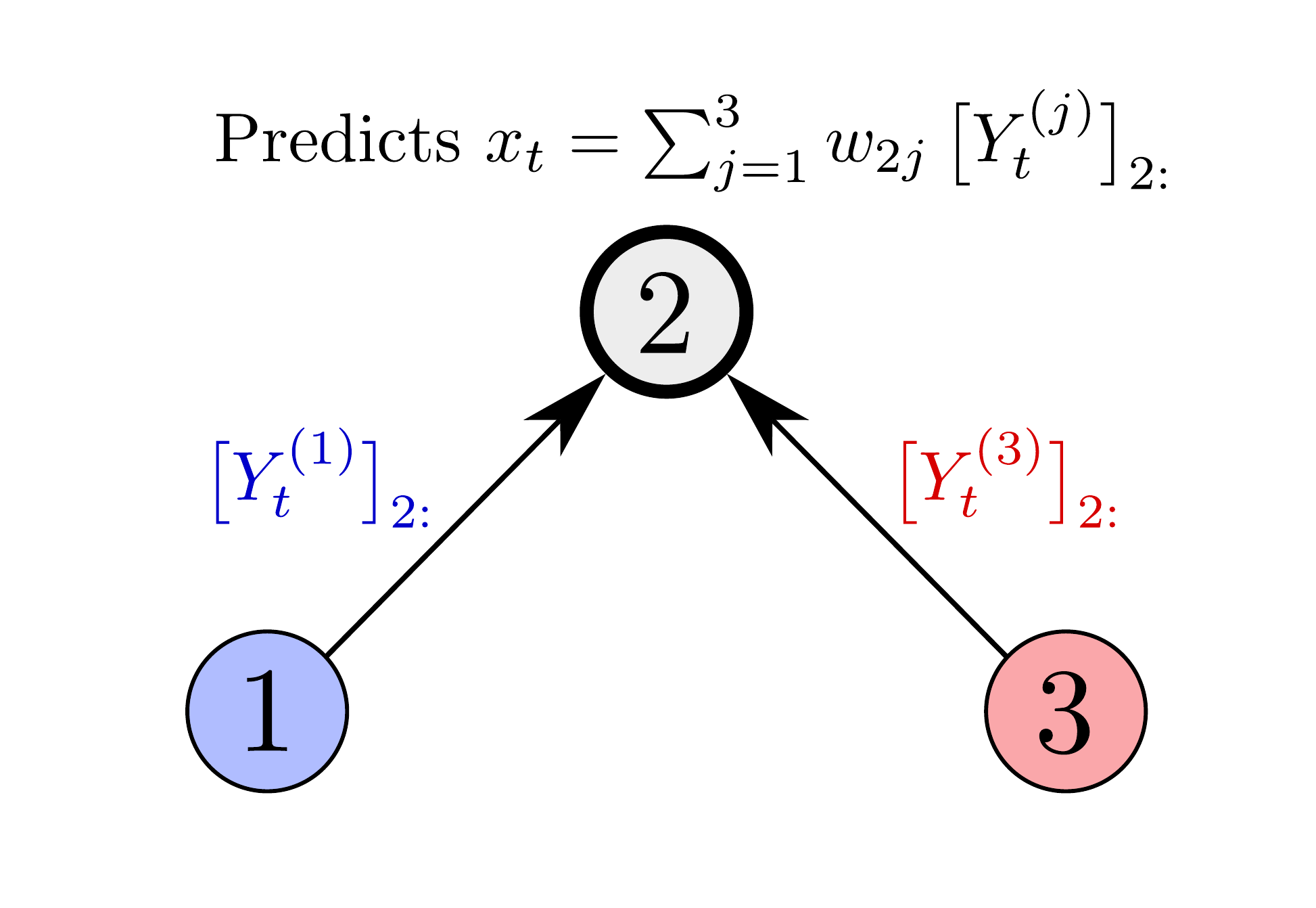}\hfill
\includegraphics[height=3.5cm]{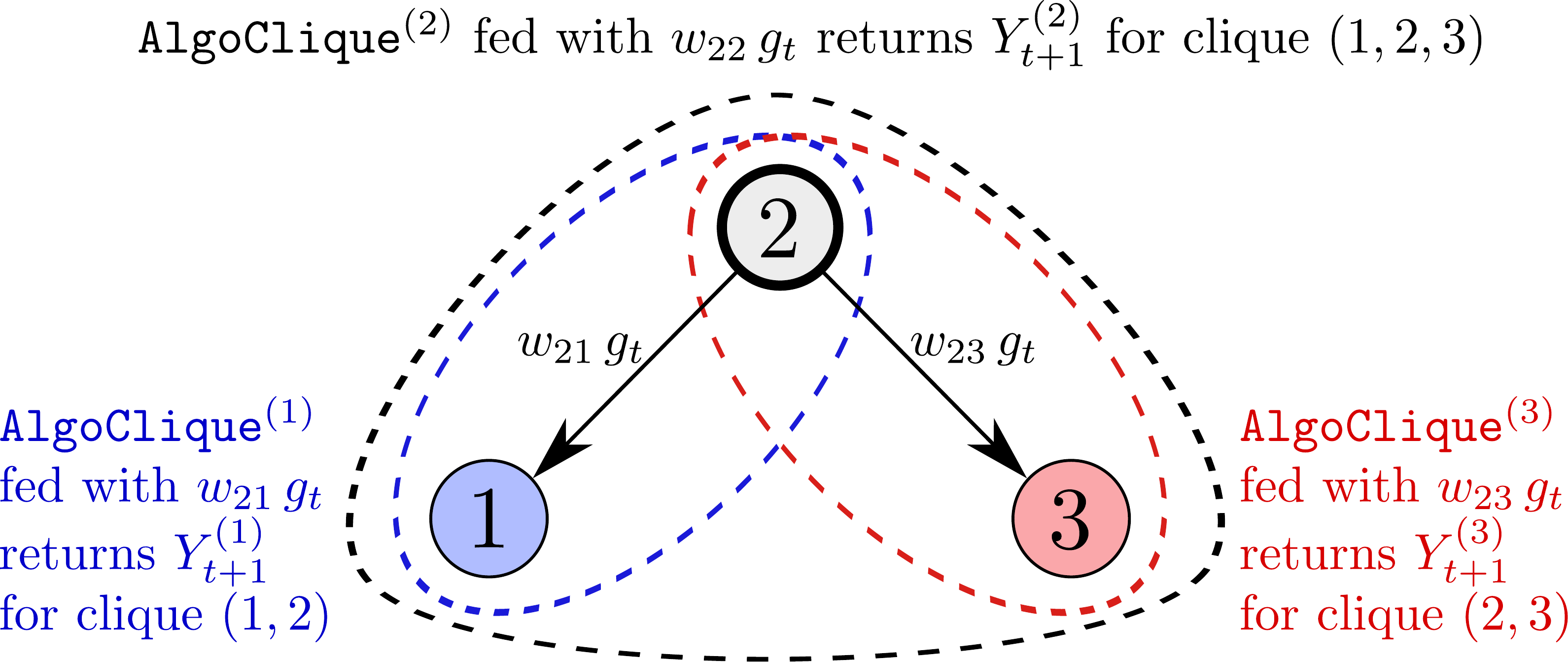}
\hspace{.3cm}
\caption{One Iteration of \Coolcn. Active agent $2$ fetches their neighbor's models and predicts with their weighted average $x_t$ (left). Then $2$ pays $\ell_t(x_t)$, observes $g_t \in \partial\ell_t(x_t)$, and sends back $w_{2j}\,g_t$. Finally, each local  $\Aclique$ instance
updates the models for the agents in the virtual clique centered on the agent running the instance (right).\looseness-1}
\label{fig:cool_cn}
\end{figure*}

In this section, we introduce and analyze \mtcool, a meta-algorithm for MultiTask COmmunication-COnstrained Online Learning.
In particular, we prove two sets of regret upper bounds, depending on whether the agent activations are adversarial or stochastic.
We also prove some lower bounds on the regret.

\Coolcn uses as building block any multitask algorithm able to operate when the communication graph is a clique. In the following, we generically refer to this base algorithm as \Aclique.
\Coolcn requires each agent $j \in V$ to run an instance of \Aclique on a virtual clique over its neighbors $i \in \mathcal{N}_j$.\footnote{Note that the alternative consisting in running \Aclique for each pair $(i,j) \in E$ fails, see \Cref{sec:fail}.}
The instance of \Aclique run by $j$ maintains a matrix $Y_t^{(j)} \in \mathbb{R}^{N_j \times d}$, whose each row $\big[Y_t^{(j)}\big]_{i:}$ stores a model associated to task $i$, for any $i \in \scN_j$.
%
At time $t$, each $j\in\scN_{i_t}$ sends to the active agent $i_t$ (or equivalently $i_t$ fetches from $j$) the model $\big[Y_t^{(j)}\big]_{i_t:}$. 
The prediction made by $i_t$ at time $t$ is the weighted average $x_t = \sum_{j \in \scN_{i_t}} w_{i_tj} \big[Y_t^{(j)}\big]_{i_t:}$, where $\wij$ are non-negative weights satisfying $\sum_{j\in \scN_{i}} \wij =1$.
After predicting, $i_t$ observes $g_t \in \partial\ell_t(x_t)$ and sends $w_{i_tj}\,g_t$ to each $j\in\scN_{i_t}$.
Agents $j\in\scN_{i_t}$ then feed the linear loss $\langle w_{i_tj} g_t, \cdot\rangle$ to their local instance of \Aclique, obtaining the updated matrix $Y_{t+1}^{(j)}$.
The pseudocode of \mtcool is summarized in \Cref{alg:global_alg}.
See \Cref{fig:cool_cn} for an illustration of one iteration.
One interpretation behind \Coolcn is as follows: by fetching the predictions maintained by their neighbors, the active agent $i_t$ actually leverages information from the neighbors of their neighbors, i.e., from agents $k \in \scN_j$, for $j \in \scN_{i_t}$.
Indeed, these agents have a direct influence $Y_t^{(j)}$, so in particular on $\big[Y_t^{(j)}\big]_{i_t}$.
By iterating this mechanism, \Coolcn allows to propagate information along the communication graph.
Note that the local instance of \Aclique run by agent $j$ does not have access to the global time $t$, e.g., to set the learning rate, but may only use the local time $\sum_{s \le t} \Ind{i_s \in \scN_j}$.

An important consequence of the fetch and send operations in \Coolcn is that they enable running \Aclique \textit{as if each agent $j$ were part of an isolated clique over $\mathcal{N}_j$}, which in turn makes \Cref{thm:mt-ftrl} applicable.
Formally, let \Coolcn be run over an arbitrary sequence $i_1,\ldots,i_T$ of agent activations in a communication network $G = (V,E)$.
Then, for each $j \in V$ and time $t$, the matrix $Y_t^{(j)}$ of models computed by \Aclique run at node $j$ is identical to the matrix computed by \Aclique run on a clique over $\scN_j$ with the sequence of activations restricted to $\scN_j$.
This simple yet key observation allows us to derive the following lemma, which shows that the regret of \Coolcn can be expressed in terms of the regrets suffered by \Aclique run on the artificial cliques $\scN_j$, for $j \in V$.
\vspace{-0.3cm}

\begin{lemma}\label{lem:cool-cn}
For any $U \in \mathbb{R}^{N \times d}$, let $U^{(j)} \in \mathbb{R}^{N_j \times d}$ be the matrix whose rows contain all $U_{i:}$ for $i$ in $\scN_j$, sorted in ascending order of $i$.
Furthermore, let $R^\textnormal{clique-$j$}_T$ be the regret suffered by \Aclique run by agent $j$ on the linear losses $\langle w_{i_tj}\,g_t, \cdot\rangle$ over the rounds $t \le T$ such that $i_t \in \mathcal{N}_j$.
Then, the regret of \Coolcn satisfies
\[
\forall\, U \in \mathbb{R}^{N \times d}, \quad R_T(U) \le \sum_{j=1}^N R^\textnormal{clique-$j$}_T\big(U^{(j)}\big)\,.
\]
\end{lemma}
\begin{proof}
Let $x_1,\ldots,x_T$ be the predictions of \Coolcn.
%
By convexity of the $\ell_t$, we have
\begin{align}
&\sum_{t=1}^T \ell_t(x_t) - \ell_t(U_{i_t:}) \le \sum_{t=1}^T \langle g_t, x_t - U_{i_t:}\rangle\label{eq:linearization}\\
&~~= \sum_{t=1}^T \sum_{i=1}^N  \Big\langle g_t, \sum_{j \in \mathcal{N}_i} \wij\big[Y_t^{(j)}\big]_{i:} -U_{i:} \Big\rangle \Ind{i_t=i}\nonumber\\
&~~= \sum_{t=1}^T \sum_{i=1}^N \sum_{j \in \mathcal{N}_i}   \left\langle g_t,   \wij\left(\big[Y_t^{(j)}\big]_{i:} -U_{i:} \right) \right\rangle \Ind{i_t=i}\nonumber\\
&~~= \sum_{j=1}^N \sum_{t=1}^T \sum_{i \in \mathcal{N}_j} \left\langle \wij\,g_t,  \big[Y_t^{(j)}\big]_{i:} -U_{i:}  \right\rangle \Ind{i_t=i}\label{eq:invert}\\
&~~= \sum_{j=1}^N R^\text{clique-$j$}_T\big(U^{(j)}\big)\,,\nonumber
\end{align}
where \eqref{eq:invert} follows from \Cref{lem:invert}.
\end{proof}

Note that we could obtain similar guarantees by using the convexity of $\ell_t$ and postponing the linearization step \eqref{eq:linearization} to the individual clique regrets.
However, this would require computing the subgradients at the partial predictions $\big[Y_t^{(j)}\big]_{i:}$.
Instead, \mtcool only uses the subgradient $g_t$ evaluated at the true prediction $x_t$.
In the next section, we show how to leverage \Cref{lem:cool-cn} to derive regret bounds for \mtcool.

\subsection{Adversarial Activations}
\label{sec:adversarial_activations}

In this section, we assume the sequence $i_1,i_2,\ldots\in V$ of agent activations is chosen by an oblivious adversary.
In the next theorem, we show that setting \Aclique as \MTFTRL yields good regret bounds.
Importantly, since the instance of \MTFTRL of agent $j$ is run on the virtual clique $\scN_j$, the regret bound scales with the \emph{local} task variances $\sigma^2_j = \frac{1}{2N_j(N_j-1)} \sum_{i, i'\in \mathcal{N}_j} \big\|U_{i:} - U_{i':}\big\|_2^2$,
quantifying how similar neighboring tasks are.
%
Let $\sigma^2_\textnormal{max} = \max_{j \in [N]} \sigma^2_j$ and $\sigma^2_\textnormal{min} = \min_{j \in [N]} \sigma^2_j$.
We are now ready to state our result.
%
\begin{restatable}{theorem}{thmanyG}\label{thm:any-G}
Let $G = (V,E)$ be any communication graph.
Consider \Coolcn where the base algorithm \Aclique run by each agent $j\in V$ is an instance of \MTFTRL with parameters $N=N_j$ and $\beta_{t-1} = \max_{i\in \scN_j} \wij \sqrt{1+\sum_{s \le t-1}\Ind{i_s \in \scN_j}}$.
Then, the regret of \Coolcn satisfies for all $U \in \mathcal{U}$
\[
R_T(U) \tildeO \sum_{j=1}^N \max_{i\in \scN_j} \wij \sqrt{1 + \sigma_j^2 (N_j-1)}\sqrt{\sum_{i \in \scN_j} T_i}\,.
\]
Setting $w_{ij} = \Ind{j\in\mathcal{N}_i}/N_i$ we obtain
%
\begin{align*}
R_T(U) \tildeO \min\Bigg\{ \!&\frac{\sqrt{NN_\textnormal{max}}}{N_\textnormal{min}}\sqrt{1 + \sigma_\textnormal{max}^2(N_\textnormal{max}-1)}\,,\\
&~\frac{N}{N_\textnormal{min}}\sqrt{1 + \bar{\sigma}^2(N_\textnormal{max}-1)} \Bigg\} \sqrt{T}\,.
\end{align*}
%
where $\bar{\sigma}^2 = (1/N)\sum_{j=1}^N\sigma_j^2$ is the average local variance.
Finally, there exist some $w_{ij}$ such that
\[
R_T(U) \tildeO \sqrt{1 + \Delta^2 (N_\textnormal{max}-1)} \sqrt{\gamma(G)\, T}\,,
\]
where $\Delta^2 = \sup_{(i,j) \in E} \|U_{i:} - U_{j:}\|_2^2$.
\end{restatable}

\Cref{thm:any-G} provides three results with different flavors.
The first bound is the most general, and holds for any choice of weights $w_{ij}$.
It shows that the regret of \mtcool improves as the $\sigma_j^2$ get smaller, i.e., when neighbors in $G$ have similar tasks.
The second bound uses uniform weights that only depend on local information (the neighborhood sizes). Interestingly, the bound depends on the total horizon $T$ instead of the individual $T_i$, and gets smaller as $G$ becomes dense ($N_\textnormal{min} \gg 1$) and regular ($N_\textnormal{max}/N_\textnormal{min} \to 1)$. 
Finally, the third bound uses knowledge of the full graph $G$ (in particular of its smallest dominating set, which is NP-hard to compute) to set the weights.
This bound can be viewed as a multitask version of the single-task regret bound $\sqrt{\gamma(G)\,T}$ (recovered for $\Delta^2=0$) achieved by the centralized algorithm that pre-computes this dominating set and then delegates the predictions of each node to its dominating node.
%
These bounds also show the improvement brought by sharing models, as methods sharing only gradients suffer from a lower bound of $\sqrt{\alpha(G)\,T}$ \citep{cesa2020cooperative}.
We now instantiate our bound for uniform weights to some particular graphs of interest.

\begin{restatable}{corollary}{corspecialgraph}\label{cor:reggraphs_cliques}
Under the assumptions of \Cref{thm:any-G}, set $w_{ij} = \Ind{j\in\mathcal{N}_i}/N_i$.
If $G$ is $K$-regular, then
\[
R_T(U) \tildeO \sqrt{1 + K\,\sigma_\textnormal{max}^2} \, \sqrt{\frac{NT}{K+1}}\,.
\]
If $G$ is a union of $\chi$ cliques, then\looseness-1
\[
R_T(U) \tildeO \sqrt{\chi + \sigma^2_{\textnormal{max}}(N - \chi)} \sqrt{T}\,.
\]
\end{restatable}

The regret bound for $K$-regular graphs improves upon the $\sqrt{\alpha(G)\,T}$ bound proven by \citet{cesa2020cooperative} in the more restrictive setting of single-task problems with stochastic activations.
Indeed, by Turan's theorem, see e.g., \citet[Lemma~3]{mannor2011bandits}, we have that $N/(K+1) \le \alpha(G)$ for $K$-regular graphs.
This shows that the fetch step in \Coolcn, which is missing in the learning protocol of \citet{cesa2020cooperative}, is key to derive improved regret guarantees.
Regarding the second bound in \Cref{cor:reggraphs_cliques}, we recover the bound of \texttt{MT-FTRL} \citep{cesa2022multitask} when $\chi=1$.
Notably, \mtcool achieves these regret bounds without requiring knowledge of the graph structure.

\subsection{Stochastic Activations}
In this section, we show that \mtcool enjoys improved regret bounds when the agent activations are stochastic rather than adversarial.
Formally, let $q_1, \ldots, q_N \in [0, 1]^N$ such that $\sum_{i=1}^N q_i = 1$.
Let $i_1,\ldots,i_T$ be independent random variables denoting agent activations
such that for every $t$ we have $\mathbb{P}(i_t = i) = q_i$.
We also define $Q_j = \sum_{i \in \scN_j} q_i$.
In the following theorem, we show how \mtcool can adapt to the stochasticity to attain better regret bounds.
For simplicity, we assume the algorithm discussed below has access to the conditional probabilities $q_i/Q_j$.
In \Cref{rmk:unknown_q}, we discuss a simple strategy to extend the analysis to the case where only a lower bound on $\min_i q_i$ is available.

\begin{restatable}{theorem}{thmsto}\label{thm:sto}
%
%
%
Let $G = (V,E)$ be any communication graph. Consider \Coolcn where the base algorithm \Aclique run by each agent $j\in V$ is an instance of \MTFTRL with parameters $N=N_j$ and $\beta_{t-1} = \sqrt{\sum_{i \in \scN_j} \frac{q_i}{Q_j}\,w_{ij}^2} \sqrt{1+\sum_{s \le t-1}\Ind{i_s \in \scN_j}}$.
Then, the regret of \Coolcn satisfies for all $U \in \mathcal{U}$
\[
\mathbb{E}[R_T(U)] \tildeO \hspace{-0.04cm}\left(\sum_{j=1}^N \sqrt{\sum_{i\in \mathcal{N}_j}q_i\,w_{ij}^2} \sqrt{1 + \sigma_j^2 (N_j-1)}\right) \hspace{-0.1cm}\sqrt{T}\,,
\]
where the expectation is taken over $i_1,\ldots,i_T$.
Setting $w_{ij} = q_j/Q_i$ we obtain
\[
\mathbb{E}[R_T(U)] \tildeO \sqrt{1+\sigma_\textnormal{max}^2(N_\textnormal{max}-1)}\sqrt{\alpha(G)\,T}\,.
\]
\end{restatable}

The improvement with respect to \Cref{thm:any-G} is a natural consequence of the stochastic activations. 
Indeed, the term $\max_{i \in \scN_j} w_{ij}$ arises because agent $j$ receives scaled gradients from its neighbors $i \in \scN_j$, with squared norms bounded by $\|w_{ij}\,g_t\|_2^2 \le \max_{i \in \scN_j} w_{ij}^2$.
With stochastic activations, this norm can be bounded in expectation by $\sum_{i \in \scN_j} \frac{q_i}{Q_j} w_{ij}^2$,
resulting in the improvement seen in in the first bound of \Cref{thm:sto}.
In the second bound, we show that an appropriate choice of weights $w_{ij}$ allows to recover the $\sqrt{\alpha(G)\,T}$ bound derived by \citet{cesa2020cooperative} in the single-task setting ($\sigma_\textnormal{max}^2 = 0$).
Importantly, the choice of $w_{ij}$ remains local to agent $j$, and does not require knowledge of the full communication graph.
Finally, while \Cref{cor:reggraphs_cliques} shows regret bounds sharper than $\sqrt{\alpha(G)\,T}$ in the more general adversarial case, 
those bounds only apply to particular graphs.
In contrast, the results of \Cref{thm:sto} are valid for any graph.
We conclude this section with a remark extending the guarantees to the case where the $q_j/Q_i$ have to be estimated.

\begin{remark}[Extension to unknown $q_i$]\label{rmk:unknown_q}
Both the choice of learning rates $\beta_{t-1}$ and weights $w_{ij}$ in the second claim of \Cref{thm:sto} require that agent $j$ can access the conditional probabilities $q_i/Q_j$.
When the latter are unknown, they can be estimated easily through $\widehat{\pi}_{ij}(t) \coloneqq  \sum_{s\le t} \Ind{i_s = i} / \sum_{s\le t} \Ind{i_s \in \scN_j}$.
In \Cref{apx:unknown_q}, we show that the 2-step strategy consisting in first computing the $\hat{\pi}_{ij}$ and then running \Coolcn with $\hat{\pi}_{ij}$ instead of $q_i/Q_j$ suffers minimal additional regret.
Note that, although appealing, the idea of replacing $q_i/Q_j$ in $\beta_{t-1}$ by $\hat{\pi}_{ij}(t)$ at each time step $t$ would not work, as it would disrupt the monotonicity of the learning rate sequence $\beta_{t-1}$.
\end{remark}

\subsection{Lower Bounds}

We conclude the regret analysis by providing lower bounds.
In particular, we show that \mtcool's regret bound for regular graphs is tight up to constant factors. \looseness-1

\begin{restatable}{theorem}{thmlowerbound}\label{thm:lower-bound}
Let $G$ be any communication graph.
Then for any algorithm following \Cref{sec:model}'s protocol:
\begin{enumerate}[nosep,wide]
\item there exists a sequence of activations and gradients such that the algorithm suffers regret
\[
\sup_{\substack{U \in \mathcal{U}\\{\sigma^2(U) \le \nu^2}}}\hspace{-0.3cm}R_T(U) \geq \max\hspace{-0.03cm}\Big\{\hspace{-0.09cm}\sqrt{1 + \nu^2 (N-1)}, \sqrt{\alpha_2(G)}\Big\}\frac{\sqrt{T}}{3}.
\]
\item for any even number $K$, there exists a $K$-regular graph and a sequence of activations and gradients such that the algorithm suffers regret
\[
\sup_{U \in \mathcal{U}} R_T(U) \geq \frac{1}{5} \sqrt{1 + K\,\sigma^2_\textnormal{min}}\,\sqrt{\frac{NT}{K}}\,.
\]
\end{enumerate}
\end{restatable}

The first result is the maximum of two lower bounds.
The bound depending on the (upper bound of the) comparator variance $\nu^2$ is due to \cite{cesa2022multitask}, and holds even without communication constraints.
The bound depending on $\alpha_2(G)$ is established in \Cref{apx:lower}, and leverages the fact that there is no information leakage if the activated nodes are always 2 edges apart.
The second result investigates this idea further for $K$-regular graphs, and provides a matching lower bound for \Cref{cor:reggraphs_cliques}'s first claim.

\section{\MakeUppercase{A private variant}}\label{sec:DP}

\Coolcn involves a gradient-sharing step, potentially harming privacy if losses contain sensitive information.
In this section, we modify \Coolcn so as to satisfy \textit{loss-level differential privacy}.
We prove that \ifthenelse{\boolean{long}}{}{ for linear losses,} privacy only degrades the regret by a term polylogarithmic in $T$.
Additionally, we establish privacy thresholds where sharing information becomes ineffective.
Our analysis is based on the following privacy notion.\looseness-1

\begin{definition}\label{def:loss-level-priv}
We consider a randomized multi-agent algorithm $\mathcal A$, governing the communication between $N$ agents.
Let $m^{(i)}_t$ denote the batch of messages sent by agent~$i$ to the other agents at time $t$.
For any $\epsilon>0$, $\mathcal A$ is loss-level $\epsilon$-differentially private (or $\epsilon$-DP for short) if for all $i\in \setN$ and set $\mathcal{M}$ of sequences of messages
 \begin{equation}\label{eq:loss-level-priv}
\frac{\sP\big(m^{(i)}_{1}, \ldots, m^{(i)}_{T} \in \mathcal{M} \mid i_1, \ell_1, \ldots, i_T, \ell_T\big)} {\sP\big(m^{(i)}_{1}, \ldots, m^{(i)}_{T}  \in \mathcal{M} \mid i_1, \ell'_1, \ldots, i_T, \ell'_T\big)} \leq e^{\epsilon}\,,
\end{equation}
where $(\ell'_t)_{t \le T}$ and $(\ell_t)_{t \le T}$ differ by at most one entry.
\end{definition}

%
As it applies to messages between agents, our definition of DP is more general than traditional notions which focus on predictions, as commonly done in single-agent scenarios \citep{jain2012differentially,agarwal2017price}.
Note that a similar definition has been employed by \citet{bellet2018personalized}, in the batch case though.

%
Our modification of \Coolcn relies on the fact that \MTFTRL only requires information on sums of gradients and sums of inner products between predictions and gradients (for \texttt{Hedge}). This allows us to use
aggregation trees \citep{dwork2010differential,chan2011private}, enabling the DP release of cumulative vector sums so that the level of noise introduced remains logarithmic in the number of vectors.
In particular, agent $i_t$ can compute: 
\vspace{-0.4cm}
\begin{itemize}[topsep=0pt,itemsep=0.5pt,wide] 
\item Sanitized versions $\tilde{\gamma}^{(i)}_t$ of the sums of gradients observed by the agents $\sum_{s\le t} g_s\,\Ind{i_s = i}$;
\item Sanitized versions $\tilde{s}_{t,i}^{(j, \xi)}$ of the sums of inner products between predictions of expert $\xi$ of agent $j$ and weighted gradients $\sum_{s \le t-1} \langle w_{ij} g_s, \big[X_s^{(j, \xi)}\big]_{i:}\rangle \Ind{i_s = i}$.
\end{itemize}
This technique was used in \cite{guha2013nearly,agarwal2017price} for DP single-agent online learning.
As in \cite{agarwal2017price}, we use \texttt{TreeBasedAgg}, which tweaks the original tree-aggregation algorithm to obtain identically distributed noise across rounds. \looseness-1

\begin{algorithm}[t]
\caption{ \DOPE}\label{alg:DOPE_learning}
\Req{Base algorithm \DPMTFTRL, weights $w_{ij}$, privacy level $\epsilon$, distributions $\mathcal{D}_d$, $\mathcal{D}_1$} \vspace{0.07cm}

\For{$t=1, 2, \ldots$}{\vspace{0.1cm}

Active agent $i_t$\vspace{0.07cm}
    
\quad fetches $\big[Y_t^{(j)}\big]_{i_t:}$,  $\Big\{\big[X_t^{(j, \xi)}\big]_{i_t:} \, : \xi \in \Xi_j\Big\}$ from each

\quad  $j \in \mathcal{N}_{i_t}$
    
\quad predicts $x_t = \sum_{j \in \mathcal{N}_{i_t}} w_{i_tj}\big[Y_t^{(j)}\big]_{i_t:}$\vspace{0.1cm}
    
\quad pays $\ell_t(x_t)$ and observes $g_t \in \partial\ell_t(x_t)$\vspace{0.14cm}

\quad updates   $\GammaOut_t^{(i_t)} =   \GammaOut_{t}^{(i_t)} + g_t$   
\vspace{0.1cm}

\quad  computes $\hatGammaOut_{t}^{(i_t)}$ using an instance of

\quad \treeBasedAgg with  distribution $\mathcal{D}_d$

\vspace{0.14cm}

\For{$j \in \mathcal{N}_{i_t}$}{

\For{$ \xi \in \Xi_j$}{
$s_{t, i_t}^{(j,\xi)} =  s_{t-1, i_t}^{(j,\xi)} + w_{i_tj}\big\langle[X_t^{(j,\xi)}]_{i_t:}, g_t\big\rangle,$ 
\vspace{0.1cm}

computes $\tilde{s}_{t, i_t}^{(j,\xi)}$ using an instance of

\treeBasedAgg set with distribution $\mathcal{D}_1$ 
    }
    sends $\big(i_t, w_{i_tj}\,\hatGammaOut_{t}^{(i_t)}, \{\tilde{s}_{t, i_t}^{(j,\xi)} : \xi\in\Xi_j\} \big)$ to $j$ 
    



Agent $j$ feeds $\big(w_{i_tj}\,\hatGammaOut_{t}^{(i_t)}, \{\tilde{s}_{t, i_t}^{(j, \xi)}  : \xi\in\Xi_j\}\big)$ to their local instance of \DPMTFTRL and obtains $\big(Y^{(j)}_{t+1}, \{X_{t+1}^{(j, \xi)},~\forall \xi \in \Xi_j\}\big)$. 

}
}
\end{algorithm}

The resulting algorithm, called \DOPE, is described in \Cref{alg:DOPE_learning}.
It invokes  the $\epsilon$-DP version of \MTFTRL, summarized in the supplementary material (\Cref{alg:mt-ftrl-DP}).
We now prove that \DOPE is $\epsilon$-DP.
\begin{restatable}{theorem}{thmprivacy}\label{th:privacy}
Let $G$ be any graph, and for any $\varepsilon > 0$, let $\epsilon' = \epsilon / (6N_\textnormal{max}^2)$.
Assume that \DOPE is run \ifthenelse{\boolean{long}}{}{on linear losses } with $\mathcal{D}_d$ set to a $d$-dimensional Laplacian distribution with parameter $\frac{\sqrt{d} \ln T}{\epsilon'}$, and $\mathcal{D}_1$ set to a $1$-dimensional Laplacian distribution with parameter $\frac{\ln T}{\epsilon'}$. Then \DOPE is $\epsilon$-DP.
\end{restatable}
Next, we state the regret bound for \DOPE with adversarial agent activations\ifthenelse{\boolean{long}}{}{ and linear losses}.

\begin{restatable}{theorem}{regretDP}
\label{th:reg-DP_anyG}
%
Let $G = (V,E)$ be any communication graph,
Consider \DOPE where the base algorithm run by each agent $j\in V$ is an instance of \DPMTFTRL with parameters $N=N_j$ and $\beta_{t-1} = \max_{i\in \scN_j} \wij \sqrt{1+\sum_{s \le t-1}\Ind{i_s \in \scN_j}}$. Then the regret of \DOPE \ifthenelse{\boolean{long}}{}{run over linear losses} with $\mathcal{D}_d$ and $\mathcal{D}_1$ set as in \Cref{th:privacy} satisfies
%
\begin{align}
\E[R_T(U)] & \tildeO \sum_{j=1}^N \max_{i\in \scN_j} \wij \sqrt{1 + \sigma_j^2(N_j - 1)} \sqrt{ \sum_{i \in \scN_j} T_i}\nonumber\\
&~~~+ \frac{dN_{\textnormal{max}}^4}{\epsilon} N\ln^2 T\,,\label{eq:privacy_cost}
\end{align}
where randomness is due to sanitization.
\end{restatable}
\Cref{th:reg-DP_anyG} shows that the additional regret due to the privatization of our algorithm (second line in bound \eqref{eq:privacy_cost}) has a very mild dependence on $T$.
In particular, when $G$ is $K$-regular, it becomes negligible as soon as $T \ge N (d^2 K^9/\varepsilon^2)$.
On the other side, the bound indicates that running \DOPE is worse than running $N$ instances of \FTRL without communication (which is DP by design) when $\epsilon \le dN_\textnormal{max}^4\sqrt{N/T}$.
In short, if agents are too private, listening to neighbors introduces more noise than valuable information.
Note that the cut-off value for $\epsilon$ below which sharing information is detrimental vanishes as $T$ goes to $+\infty$.
Hence, for any privacy level $\epsilon$, \DOPE is asymptotically (in $T$) preferable to independent instances of \FTRL.
Note finally that the cost of privacy remains the same for stochastic activations, and that in the single-agent case we do recover the result of \citet{agarwal2017price}.

\ifthenelse{\boolean{long}}{Another possibility to maintain DP is to }{The assumption of linear losses  is crucial for \Cref{th:privacy}, as it ensures that predictions do not affect gradients.
If losses are convex but not necessarily linear, we can still} implement \DOPE by adding Laplacian noise of parameter $1/\epsilon$ to each individual gradient.
It can be shown that this method yields a total privacy cost of order at least $d\left(\sum_{j =1}^N \max_{i\in \scN_j} \wij \sqrt{\sum_{i \in \mathcal{N}_j} T_i}\right)/{\epsilon}$.
In particular, when $G$ is $K$-regular, the cut-off value for $\epsilon$ becomes $d/ \sqrt{K}$, which does not vanish as $T\to\infty$.
Hence, this way of privatizing \mtcool is too coarse to maintain better performances over non-communicating methods for all privacy levels $\epsilon$.

Interestingly, the proof of \Cref{th:reg-DP_anyG} also shows that knowing the task variances $\sigma_j^2$ reduces the cost of privacy, because the instances of \MTFTRL do not need to use experts to adapt to $\sigma_j^2$. 
%
%
In particular, the active agent can avoid sharing $\tilde{s}_{t, i_t}^{(j,\xi)}$, which leads to a significant reduction in the information leakage.
%
%
Then, the cost of privacy becomes of order $\Big(\sum_{j=1}^N \sqrt{N_j\big(1 + \sigma_j^2(N_j - 1)\big)} \Big) d\ln^2 T/\epsilon$.

\section{\MakeUppercase{Experiments}}\label{sec:expe}

In this section, we describe experiments supporting our theoretical claims.
In our experiments, the communication graphs $G$ are generated using the Erdős–Rényi's model: given $N=30$ vertices, the edge between two vertices is included in $G$ with probability $p=0.9$.
Our results are averaged over $48$ independent draws of such graphs.
The losses are of the form $\ell_t(x) = \frac{1}{2}\|x - z_t\|_2^2$, where the $z_t$ are realizations of a multivariate Gaussian with covariance matrix $10^{-4}\,I_d$ and expectation $U_{i_t:}$.
Here, $i_t$ denotes the active task, and $U_{i_t:}$ the corresponding task vector.
The matrix $U$ is drawn according to a centered Gaussian distribution with covariance matrix $(I_N+ \lambda L_G)^{-1}$, where $L_G$ is the Laplacian of the graph $G$. 
The local task variance is then controlled via the parameter $\lambda$, which in our experiments varies in the set $\{10^{10}, 10, 9, 8, 7, 6, 5, 4, 3, 2\}$.\footnote{
We mimicked $\lambda=\infty$, i.e., no task variance, by $\lambda = 10^{10}$.}
The average local task standard deviation $\bar{\sigma}$ resulting from these choices of $\lambda$ are the $x$-coordinates of points in \Cref{fig:expe_2}.
Activations are stochastically generated with $q_i= 1/N$. 

For practical reasons, the version of \mtcool used here slightly differs from \Cref{alg:global_alg}, as our adaptive part is based on the Krichevsky-Trofimov algorithm (see \Cref{alg:mt-ftrl-kt} for details) to account for task variances and domain diameter.
This variant has regret guarantees similar to the original \mtcool (see \cref{sec:app_expe}) but is simpler to run in practice.
It is however nontrivial to make it $\epsilon$-DP, which is the reason why in the theoretical analysis we privileged \texttt{Hedge} over Krichevsky-Trofimov.
We consider two baselines: the first one, \texttt{i-FTRL}, consists in running instances of \FTRL on agents that do not communicate with each other. 
The second method, \texttt{ST-FTRL}, is the multi-agent single-task algorithm with communication graph $G$ introduced in \citep{cesa2020cooperative}.
Note that this algorithm requires oracle access to the loss function, instead of just access to the gradient computed at the current prediction. This represents a significant advantage. 
In particular, the regret guarantees for this algorithm only hold given access to the loss function oracle. 
To ensure fair comparisons, we extend adaptivity to the diameter to both \texttt{i-FTRL} and \texttt{ST-FTRL} using the Krichevsky-Trofimov method. \looseness -1

\begin{figure}[t]
\centering
\includegraphics[width=0.81\columnwidth]{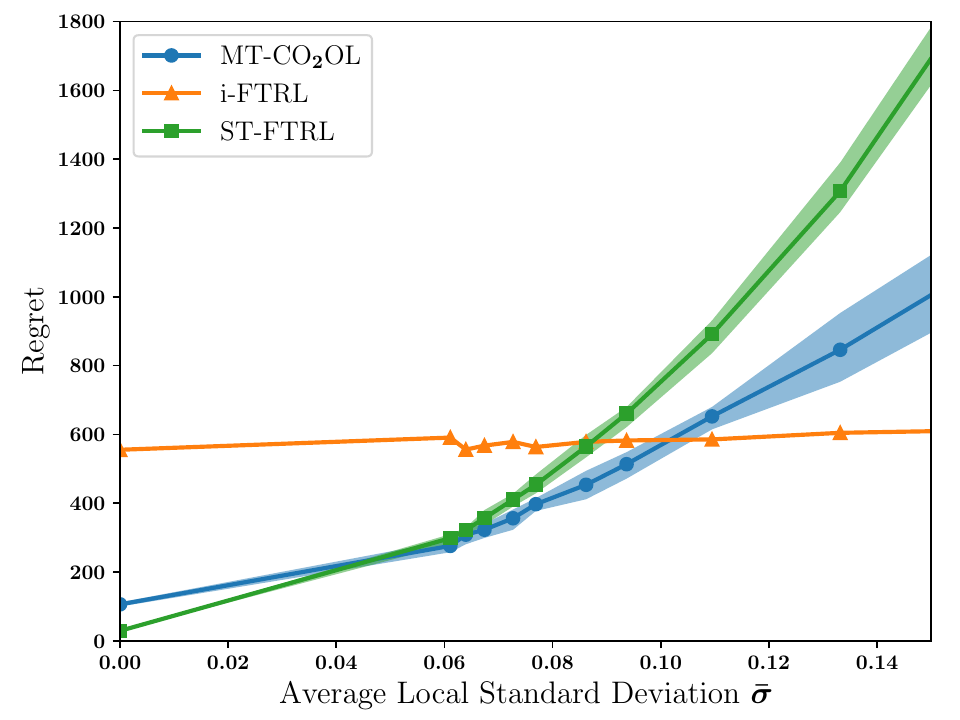}
\caption{Multitask regret at horizon $T=150\,000$.}
\label{fig:expe_2}
\end{figure}

\Cref{fig:expe_2} shows the (multitask) regret of these algorithms. 
As expected, the performance of \texttt{i-FTRL} is independent from $\bar{\sigma}$ and dominates when the latter is large (i.e., tasks are very different). On the other hand, \texttt{ST-FTRL} wins when $\bar{\sigma}$ is close to zero (i.e., tasks are very similar). Finally, \mtcool outperforms the baselines for intermediate values of $\bar{\sigma}$ (see also \cref{fig:expe_1} that plots the regret against time for $\bar{\sigma} = 0.08$).

\section{\MakeUppercase{Conclusion}}

We introduced and analyze \mtcool, an algorithm tackling multitask online learning on arbitrary communication networks.
Interesting generalizations of this work may focus on the cases where: (1) the communication network changes over time, (2) privacy levels are user-specific, (3) agents have further constraints (e.g., size-wise or frequency-wise) on the messages they send.\looseness-1

\subsubsection*{Acknowledgements}
The authors acknowledge the financial support from the MUR PRIN grant 2022EKNE5K (Learning in Markets and Society), the NextGenerationEU program within the PNRR-PE-AI scheme (project FAIR), the EU Horizon CL4-2021-HUMAN-01 research and innovation action under grant agreement 101070617 (project ELSA).
The authors thank anonymous reviewer 2 for insightful comments that helped improve the correctness of the paper.


\bibliography{biblioDL}

\begin{thebibliography}{}

\bibitem[Agarwal and Singh, 2017]{agarwal2017price}
Agarwal, N. and Singh, K. (2017).
\newblock The price of differential privacy for online learning.
\newblock In {\em International Conference on Machine Learning}, pages 32--40.
  PMLR.

\bibitem[Awerbuch and Kleinberg, 2008]{awerbuch2008competitive}
Awerbuch, B. and Kleinberg, R. (2008).
\newblock Competitive collaborative learning.
\newblock {\em Journal of Computer and System Sciences}, 74(8):1271--1288.

\bibitem[Bar-On and Mansour, 2019]{bar2019individual}
Bar-On, Y. and Mansour, Y. (2019).
\newblock Individual regret in cooperative nonstochastic multi-armed bandits.
\newblock {\em Advances in Neural Information Processing Systems}, 32.

\bibitem[Bellet et~al., 2018]{bellet2018personalized}
Bellet, A., Guerraoui, R., Taziki, M., and Tommasi, M. (2018).
\newblock Personalized and private peer-to-peer machine learning.
\newblock In {\em International Conference on Artificial Intelligence and
  Statistics}, pages 473--481. PMLR.

\bibitem[Cavallanti et~al., 2010]{cavallanti2010linear}
Cavallanti, G., Cesa-Bianchi, N., and Gentile, C. (2010).
\newblock Linear algorithms for online multitask classification.
\newblock {\em The Journal of Machine Learning Research}, 11:2901--2934.

\bibitem[Cesa-Bianchi et~al., 2020]{cesa2020cooperative}
Cesa-Bianchi, N., Cesari, T., and Monteleoni, C. (2020).
\newblock Cooperative online learning: Keeping your neighbors updated.
\newblock In {\em Algorithmic learning theory}, pages 234--250. PMLR.

\bibitem[Cesa-Bianchi et~al., 2019]{cesa2019delay}
Cesa-Bianchi, N., Gentile, C., Mansour, Y., et~al. (2019).
\newblock Delay and cooperation in nonstochastic bandits.
\newblock {\em Journal of Machine Learning Research}, 20(17):1--38.

\bibitem[Cesa-Bianchi et~al., 2022]{cesa2022multitask}
Cesa-Bianchi, N., Laforgue, P., Paudice, A., et~al. (2022).
\newblock Multitask online mirror descent.
\newblock {\em Transactions on Machine Learning Research}.

\bibitem[Chan et~al., 2011]{chan2011private}
Chan, T.-H.~H., Shi, E., and Song, D. (2011).
\newblock Private and continual release of statistics.
\newblock {\em ACM Transactions on Information and System Security (TISSEC)},
  14(3):1--24.

\bibitem[Della~Vecchia and Cesari, 2021]{della2021efficient}
Della~Vecchia, R. and Cesari, T. (2021).
\newblock An efficient algorithm for cooperative semi-bandits.
\newblock In {\em Algorithmic Learning Theory}, pages 529--552. PMLR.

\bibitem[Dwork et~al., 2010]{dwork2010differential}
Dwork, C., Naor, M., Pitassi, T., and Rothblum, G.~N. (2010).
\newblock Differential privacy under continual observation.
\newblock In {\em Proceedings of the forty-second ACM symposium on Theory of
  computing}, pages 715--724.

\bibitem[Griggs, 1983]{griggs1983lower}
Griggs, J.~R. (1983).
\newblock Lower bounds on the independence number in terms of the degrees.
\newblock {\em Journal of Combinatorial Theory, Series B}, 34(1):22--39.

\bibitem[Guha~Thakurta and Smith, 2013]{guha2013nearly}
Guha~Thakurta, A. and Smith, A. (2013).
\newblock (nearly) optimal algorithms for private online learning in
  full-information and bandit settings.
\newblock {\em NeurIPS}, 26.

\bibitem[Herbster et~al., 2021]{herbster2021gang}
Herbster, M., Pasteris, S., Vitale, F., and Pontil, M. (2021).
\newblock A gang of adversarial bandits.
\newblock {\em Advances in Neural Information Processing Systems},
  34:2265--2279.

\bibitem[Hosseini et~al., 2013]{hosseini2013online}
Hosseini, S., Chapman, A., and Mesbahi, M. (2013).
\newblock Online distributed optimization via dual averaging.
\newblock In {\em 52nd IEEE Conference on Decision and Control}, pages
  1484--1489. IEEE.

\bibitem[Hsieh et~al., 2022]{hsieh2022multi}
Hsieh, Y.-G., Iutzeler, F., Malick, J., and Mertikopoulos, P. (2022).
\newblock Multi-agent online optimization with delays: Asynchronicity,
  adaptivity, and optimism.
\newblock {\em The Journal of Machine Learning Research}, 23(1):3377--3425.

\bibitem[Ito et~al., 2020]{ito2020delay}
Ito, S., Hatano, D., Sumita, H., Takemura, K., Fukunaga, T., Kakimura, N., and
  Kawarabayashi, K.-I. (2020).
\newblock Delay and cooperation in nonstochastic linear bandits.
\newblock {\em Advances in Neural Information Processing Systems},
  33:4872--4883.

\bibitem[Jain et~al., 2012]{jain2012differentially}
Jain, P., Kothari, P., and Thakurta, A. (2012).
\newblock Differentially private online learning.
\newblock In {\em Conference on Learning Theory}, pages 24--1.

\bibitem[Jiang et~al., 2021]{jiang2021asynchronous}
Jiang, J., Zhang, W., Gu, J., and Zhu, W. (2021).
\newblock Asynchronous decentralized online learning.
\newblock {\em Advances in Neural Information Processing Systems},
  34:20185--20196.

\bibitem[Li et~al., 2019]{li2019}
Li, R., Ma, F., Jiang, W., and Gao, J. (2019).
\newblock Online federated multitask learning.
\newblock In {\em Proceedings of the 7th IEEE International Conference on Big
  Data}, pages 215--220.

\bibitem[Mannor and Shamir, 2011]{mannor2011bandits}
Mannor, S. and Shamir, O. (2011).
\newblock From bandits to experts: On the value of side-observations.
\newblock {\em Advances in Neural Information Processing Systems}, 24.

\bibitem[Mitzenmacher and Upfal, 2005]{mitzenmacher2005probability}
Mitzenmacher, M. and Upfal, E. (2005).
\newblock {\em Probability and computing: Randomization and probabilistic
  techniques in algorithms and data analysis}.
\newblock Cambridge university press.

\bibitem[Murugesan et~al., 2016]{murugesan2016}
Murugesan, K., Liu, H., Carbonell, J., and Yang, Y. (2016).
\newblock Adaptive smoothed online multi-task learning.
\newblock In {\em Proceedings of the 29th Annual Conference on Advances in
  Neural Information Processing Systems}, pages 4296--4304.

\bibitem[Orabona, 2019]{orabona2019modern}
Orabona, F. (2019).
\newblock A modern introduction to online learning.
\newblock {\em arXiv preprint arXiv:1912.13213}.

\bibitem[Saha et~al., 2011]{saha2011}
Saha, A., Rai, P., Daum{\'e}, H., Venkatasubramanian, S., et~al. (2011).
\newblock Online learning of multiple tasks and their relationships.
\newblock In {\em Proceedings of the 14th International Conference on
  Artificial Intelligence and Statistics}, pages 643--651.

\bibitem[Sinha and Vaze, 2023]{sinha2023playing}
Sinha, A. and Vaze, R. (2023).
\newblock Playing in the dark: No-regret learning with adversarial constraints.
\newblock {\em arXiv preprint arXiv:2310.18955}.

\bibitem[Smith et~al., 2017]{smith2017federated}
Smith, V., Chiang, C.-K., Sanjabi, M., and Talwalkar, A.~S. (2017).
\newblock Federated multi-task learning.
\newblock {\em NeurIPS}, 30.

\bibitem[Van~der Hoeven et~al., 2022]{van2022distributed}
Van~der Hoeven, D., Hadiji, H., and van Erven, T. (2022).
\newblock Distributed online learning for joint regret with communication
  constraints.
\newblock In {\em International Conference on Algorithmic Learning Theory},
  pages 1003--1042. PMLR.

\bibitem[Yan et~al., 2012]{yan2012distributed}
Yan, F., Sundaram, S., Vishwanathan, S., and Qi, Y. (2012).
\newblock Distributed autonomous online learning: Regrets and intrinsic
  privacy-preserving properties.
\newblock {\em IEEE Transactions on Knowledge and Data Engineering},
  25(11):2483--2493.

\bibitem[Yi and Vojnovi{\'c}, 2023]{yi2023doubly}
Yi, J. and Vojnovi{\'c}, M. (2023).
\newblock Doubly adversarial federated bandits.
\newblock {\em arXiv preprint arXiv:2301.09223}.

\bibitem[Zhang et~al., 2018]{zhang2018}
Zhang, C., Zhao, P., Hao, S., Soh, Y.~C., Lee, B.~S., Miao, C., and Hoi, S.~C.
  (2018).
\newblock Distributed multi-task classification: a decentralized online
  learning approach.
\newblock {\em Machine Learning}, 107(4):727--747.

\end{thebibliography}
\section*{Checklist}



 \begin{enumerate}

 \item For all models and algorithms presented, check if you include:
 \begin{enumerate}
   \item A clear description of the mathematical setting, assumptions, algorithm, and/or model. [Yes]
   \item An analysis of the properties and complexity (time, space, sample size) of any algorithm. [Yes]
   \item (Optional) Anonymized source code, with specification of all dependencies, including external libraries. [No]
 \end{enumerate}

 \item For any theoretical claim, check if you include:
 \begin{enumerate}
   \item Statements of the full set of assumptions of all theoretical results. [Yes]
   \item Complete proofs of all theoretical results. [Yes]
   \item Clear explanations of any assumptions. [Yes]     
 \end{enumerate}

 \item For all figures and tables that present empirical results, check if you include:
 \begin{enumerate}
   \item The code, data, and instructions needed to reproduce the main experimental results (either in the supplemental material or as a URL). [No]
   \item All the training details (e.g., data splits, hyperparameters, how they were chosen). [Yes]
         \item A clear definition of the specific measure or statistics and error bars (e.g., with respect to the random seed after running experiments multiple times). [Yes]
         \item A description of the computing infrastructure used. (e.g., type of GPUs, internal cluster, or cloud provider). [Not Applicable]
 \end{enumerate}
 The code will be released upon acceptance.

 \item If you are using existing assets (e.g., code, data, models) or curating/releasing new assets, check if you include:
 \begin{enumerate}
   \item Citations of the creator If your work uses existing assets. [Not Applicable]
   \item The license information of the assets, if applicable. [Not Applicable]
   \item New assets either in the supplemental material or as a URL, if applicable. [Not Applicable]
   \item Information about consent from data providers/curators. [Not Applicable]
   \item Discussion of sensible content if applicable, e.g., personally identifiable information or offensive content. [Not Applicable]
 \end{enumerate}

 \item If you used crowdsourcing or conducted research with human subjects, check if you include:
 \begin{enumerate}
   \item The full text of instructions given to participants and screenshots. [Not Applicable]
   \item Descriptions of potential participant risks, with links to Institutional Review Board (IRB) approvals if applicable. [Not Applicable]
   \item The estimated hourly wage paid to participants and the total amount spent on participant compensation. [Not Applicable]
 \end{enumerate}

 \end{enumerate}
\appendix
\onecolumn
\section{Technical Proofs (Results from Section \ref{sec:comm-graph})}

In this section, we gather the technical proofs of the results exposed in \Cref{sec:comm-graph}.
We start by stating a lemma which allows to invert indices when summation is made over the edges of $G$.
This result is in particular used to prove \Cref{lem:cool-cn}.
We then prove all results stated in \Cref{sec:comm-graph}.
Finally, we analyze an unsuccessful approach that runs $|E|$ instances of \MTFTRL that maintain predictions for all pair of agents $(i,j)$ such that $(i,j)\in E$.


\begin{lemma}\label{lem:invert}
Let $F \in \mathbb{R}^{N \times N}$ be any matrix (i.e., not necessarily symmetric). Then we have
\[
\sum_{i=1}^N \sum_{j \in \scN_i} F_{ij} = \sum_{j=1}^N \sum_{i \in \scN_j} F_{ij}\,.
\]
\end{lemma}

\begin{proof}
We have
\[
\sum_{i=1}^N \sum_{j \in \scN_i} F_{ij} = \sum_{i=1}^N \sum_{j=1}^N F_{ij} ~ \Ind{j \in \scN_i} =\sum_{i=1}^N \sum_{j=1}^N F_{ij} ~ \Ind{i \in \scN_j} = \sum_{j=1}^N \sum_{i \in \scN_j} F_{ij}\,.
\]
\end{proof}


\subsection{Proof of Theorem \ref{thm:any-G}}

\thmanyG*

\begin{proof}
By \Cref{lem:cool-cn}, we have
\[
R_T(U) \le \sum_{j=1}^N R^\textnormal{clique-$j$}_T\big(U^{(j)}\big)\,,
\]
where $R^\textnormal{clique-$j$}_T$ the regret suffered by \MTFTRL on the linear losses $\langle w_{i_tj}\,g_t, \cdot\rangle$ over the rounds $t \le T$ such that $i_t \in \mathcal{N}_j$.
We now upper bound each of these terms individually.
Let $j \in \setN$, using the notation in \Cref{alg:mt-ftrl} we have
\begin{align}
R^\textnormal{clique-$j$}_T\big(U^{(j)}\big) &= \sum_{t=1}^T \sum_{i \in \mathcal{N}_j} \left\langle \wij\,g_t,  \big[Y_t^{(j)}\big]_{i:} - U^{(j)}_{i:}  \right\rangle \Ind{i_t=i}\nonumber\\
&= \sum_{t \colon i_t \in \mathcal{N}_j} \left\langle w_{i_tj}\,g_t,  \big[Y_t^{(j)}\big]_{i_t:} - U^{(j)}_{i_t:}\right\rangle\nonumber\\
&= \sum_{t \colon i_t \in \mathcal{N}_j} \left\langle w_{i_tj}\,g_t,  \big[X_t\big]_{i_t:} - U^{(j)}_{i_t:}\right\rangle\label{eq:useful_gamma}\\
&\le \underbrace{\sum_{t \colon i_t \in \mathcal{N}_j} \left\langle w_{i_tj}\,g_t,  \big[X_t\big]_{i_t:} - \big[X_t^{(\xi^*)}\big]_{i_t:}\right\rangle}_{\text{Regret of \texttt{Hedge}}} + \underbrace{\sum_{t \colon i_t \in \mathcal{N}_j} \left\langle w_{i_tj}\,g_t,  \big[X_t^{(\xi^*)}\big]_{i_t:} - U^{(j)}_{i_t:}\right\rangle}_{\text{Regret with choice $\xi^*$}}\,,\label{eq:decompo_hedge}
\end{align}
where $\xi^* = \argmin_{\xi \in\,\Xi_j} \, \sum_{t \colon i_t \in \mathcal{N}_j} \Big\langle w_{i_tj}\,g_t,  \big[X_t^{(\xi)}\big]_{i_t:}\Big\rangle$.
We start by upper bounding the regret due to \texttt{Hedge}.
Let $\text{loss}_t \in \mathbb{R}^{N_j}$ be the vector storing the $w_{i_tj} \big\langle g_t, \big[X_t^{(\xi)}\big]_{i_t:}\big\rangle$ for $\xi \in \Xi_j$, and $e^* \in \mathbb{R}^{N_j}$ the one-hot vector with an entry of $1$ at expert $\xi^*$.
By the analysis of \texttt{Hedge} with regularizers $\psi_t(\bm{p}) = \frac{\beta\sqrt{1+\sum_{s\le t-1}\Ind{i_s \in \scN_j}}}{\sqrt{\ln N_j}} \sum_{k=1}^{N_j} p_k \ln(p_k)$, see e.g., \citet[Section~7.5]{orabona2019modern}, we have
\begin{align}
\sum_{t \colon i_t \in \mathcal{N}_j} \left\langle w_{i_tj}\,g_t,  \big[X_t\big]_{i_t:} - \big[X_t^{(\xi^*)}\big]_{i_t:}\right\rangle &= \sum_{t \colon i_t \in \mathcal{N}_j} \big\langle\,\text{loss}_t,  \bm{p}_t - e^*\big\rangle\nonumber\\
&\le \beta \sqrt{\ln N_j \sum_{i \in \scN_j} T_i} + \frac{\sqrt{\ln N_j}}{2\beta} \sum_{t\colon i_t \in \scN_j} \frac{\|\text{loss}_t\|_\infty^2}{\sqrt{1+\sum_{s \le t-1} \Ind{i_s \in \scN_j}}}\nonumber\\
&\le \beta \sqrt{\ln N_j \sum_{i \in \scN_j} T_i} + \frac{\sqrt{\ln N_j}}{2\beta} \sum_{t\colon i_t \in \scN_j} \frac{w_{i_tj}^2}{\sqrt{1+\sum_{s \le t-1} \Ind{i_s \in \scN_j}}}\nonumber\\
&\le \sqrt{\ln N_j \sum_{i \in \scN_j} T_i} \left(\beta + \frac{1}{\beta}\,\max_{i \in \scN_j} w_{ij}^2 \right)\nonumber\\
&= 2 \max_{i \in \scN_j} w_{ij} \sqrt{\ln N_j \sum_{i \in \scN_j} T_i}\label{eq:bound_hedge}
\end{align}
if we choose $\beta = \max_{i \in \scN_j} w_{ij}$.
We now turn to the second regret.
Assume first that $\sigma_j^2 \le 1$.
Then, the regret with choice $\xi^*$ is in particular better than the regret with choice $\bar{\xi} \in \Xi_j$ such that $\bar{\xi}-\frac{1}{N_j} \le \sigma_j^2 \le \bar{\xi}$.
Recall that the sequence $X_t^{(\bar{\xi})}$ is generated by \FTRL with the sequence of regularizers $\frac{1}{2}\|\cdot\|_{A_j}^2/\eta_{t-1}^{(\bar{\xi})}$.
By the analysis of \FTRL, see e.g., \citet[Corollary~7.9]{orabona2019modern}, we have
\begin{align}
\sum_{t \colon i_t \in \mathcal{N}_j} \Big\langle w_{i_tj}\,g_t,  \big[&X_t^{(\bar{\xi})}\big]_{i_t:} - U^{(j)}_{i_t:}\Big\rangle\nonumber\\
&\le \frac{\big\|U^{(j)}\big\|_{A_j}^2}{2\,\eta_{\sum_{i \in \scN_j}T_i-1}^{(\bar{\xi})}} + \frac{1}{2}\sum_{t\colon i_t \in \scN_j} \eta_{t-1}^{(\bar{\xi})}\big\|w_{i_tj}\,G_t\big\|_{A_j^{-1}}^2\nonumber\\[0.1cm]
&\le \frac{1+\sigma_j^2(N_j-1)}{2\sqrt{1 + \bar{\xi}(N_j-1)}}~\max_{i \in \scN_j} w_{ij}\sqrt{\sum_{i \in \scN_j} T_i} + \frac{N_j\sqrt{1 + \bar{\xi}(N_j-1)}}{2\max_{i \in \scN_j} w_{ij}} \sum_{t \colon i_t \in \scN_j} \frac{w_{i_tj}^2\,\big[A_j^{-1}\big]_{i_ti_t}\,\|g_t\|_2^2}{\sqrt{1+\sum_{s \le t-1} \Ind{i_s \in \scN_j}}}\nonumber\\[0.25cm]
&\le \frac{5}{2}\,\max_{i \in \scN_j} w_{ij}\,\sqrt{1 + \bar{\xi}(N_j-1)}\sqrt{\sum_{i \in \scN_j} T_i}\nonumber\\
&\le \frac{5}{2}\,\max_{i \in \scN_j} w_{ij}\,\sqrt{1 + \left(\sigma_j^2+\frac{1}{N_j}\right)(N_j-1)}\sqrt{\sum_{i \in \scN_j} T_i}\nonumber\\
&\le 4\,\max_{i \in \scN_j} w_{ij}\, \sqrt{1 + \sigma_j^2(N_j-1)}\sqrt{\sum_{i \in \scN_j} T_i}\,,\label{eq:bound_ftrl}
\end{align}
where we used $[{A_{j}}^{-1}]_{ii}= \frac{2}{N_j+1}$ (see for example computations in Appendix A.2 of \cite{cesa2022multitask}) for the third inequality.
Assume now that $\sigma_j^2 \ge 1$.
Then, the regret with choice $\xi^*$ is in particular better than the regret with choice $1$.
The latter corresponds to independent learning \citep{cesa2022multitask} and an analysis similar to the one above shows that its regret is bounded by
\begin{equation}
\max_{i \in \scN_j} w_{ij}\,\sqrt{N_j \sum_{i \in \scN_j}T_i} \le \max_{i \in \scN_j} w_{ij} \sqrt{1 + \sigma_j^2(N_j -1)}\sqrt{\sum_{i \in \scN_j}T_i}\,.\label{eq:bound_ftrl_2}
\end{equation}
Substituting \eqref{eq:bound_hedge} and \eqref{eq:bound_ftrl} or \eqref{eq:bound_ftrl_2} (depending on the value of $\sigma_j^2$) into \eqref{eq:decompo_hedge}, we obtain
\[
R_T(U) \le 6 \sum_{j=1}^N \max_{i \in \scN_j} w_{ij} \left(\sqrt{1 + \sigma_j^2(N_j - 1)} + \ln N_j \right)\sqrt{ \sum_{i \in \scN_j} T_i}\,.
\]

For the second claim, substituting $w_{ij} = \Ind{j \in \scN_i}/N_i$ yields
\begin{align}
\sum_{j=1}^N \max_{i \in \scN_j} w_{ij} \sqrt{1 + \sigma_j^2(N_j - 1)} \sqrt{\sum_{i \in \scN_j} T_i} &= \sum_{j=1}^N \frac{\sqrt{1 + \sigma_j^2(N_j - 1)}}{\min_{i \in \scN_j} N_i} \sqrt{\sum_{i \in \scN_j} T_i}\label{eq:staring_point}\\
&\le \frac{\sqrt{1 + \sigma_\textnormal{max}^2(N_\textnormal{max}-1)}} {N_\textnormal{min}} \sum_{j=1}^N  \sqrt{\sum_{i \in \scN_j} T_i}\nonumber\\
&\le \frac{\sqrt{1 + \sigma_\textnormal{max}^2(N_\textnormal{max}-1)}} {N_\textnormal{min}} \sqrt{N \sum_{j=1}^N \sum_{i \in \scN_j} T_i}\label{eq:Jensennn}\\
&= \frac{\sqrt{1 + \sigma_\textnormal{max}^2(N_\textnormal{max}-1)}} {N_\textnormal{min}} \sqrt{N \sum_{i=1}^N \sum_{j \in \scN_i} T_i}\nonumber\\
&\le \frac{\sqrt{N N_\textnormal{max}}}{N_\textnormal{min}} \sqrt{1 + \sigma_\textnormal{max}^2(N_\textnormal{max}-1)} \sqrt{T}\,,\label{eq:bound1_claim2}
\end{align}
where \eqref{eq:Jensennn} comes from Jensen's inequality, and the following equality from \Cref{lem:invert}.
Starting from \eqref{eq:staring_point} again, we also have
\begin{equation}
\sum_{j=1}^N \frac{\sqrt{1 + \sigma_j^2(N_j - 1)}}{\min_{i \in \scN_j} N_i} \sqrt{\sum_{i \in \scN_j} T_i} \le \sum_{j=1}^N \sqrt{1 + \sigma_j^2(N_\textnormal{max} - 1)} \frac{\sqrt{T}}{N_\textnormal{min}} \le \frac{N}{N_\textnormal{min}}\sqrt{1 + \bar{\sigma}^2(N_\textnormal{max}-1)}\sqrt{T}\,,\label{eq:bound2_claim2}
\end{equation}
where we have used Jensen's inequality, and $\bar{\sigma}^2 = (1/N)\sum_{j=1}^N\sigma_j^2$ is the average local variance.
Combining \eqref{eq:bound1_claim2} and \eqref{eq:bound2_claim2} gives the second claim of \Cref{thm:any-G}.

To prove the third claim, let $S_\gamma(G)$ be a smallest dominant set of $G$.
For each $i \in [N]$, let $j(i)$ be the node in $\scN_i$ that belongs to $S_\gamma(G)$ (if there are several, take the one with the smallest index).
We set $w_{ij} = \delta_{jj(i)}$, i.e., agent $i$ completely delegates its prediction to its neighbour in the dominant set.
Substituting $w_{ij} = \delta_{jj(i)}$ into \eqref{eq:useful_gamma}, we obtain
\[
R_T^\textnormal{clique-$j$}(U^{(j)}) = \sum_{t\colon i_t \in \mathcal{V}_j} \big\langle g_t, [X_t]_{i_t:} - U_{i_t:}^{(j)}\big\rangle\,,
\]
where $\mathcal{V}_j = \{i \in [N] \colon j(i) = j\}$ is the set of agents with $j$ as referent node.
Unrolling the proof of the first claim (substituting $\mathcal{N}_j$ by $\mathcal{V}_j$ and $w_{ij}$ by $\delta_{jj(i)}$), we get
\[
R_T(U) \overset{\tilde{\mathcal{O}}}{=} \sum_{j = 1}^N \max_{i \in \mathcal{V}_j} \delta_{jj(i)} \sqrt{1 + \tilde{\sigma}_j^2(|\mathcal{V}_j|-1)}\sqrt{\sum_{i \in \mathcal{V}_j} T_i} = \sum_{j \in S_\gamma(G)} \sqrt{1 + \tilde{\sigma}_j^2(|\mathcal{V}_j|-1)} \sqrt{\sum_{i \colon j(i) = i} T_i}\,,
\]
where $\tilde{\sigma}_j^2$ is the local variance at $j$ computed among its neighbors in $\mathcal{V}_j$ (rather than  $\mathcal{N}_j$).
Note that $\tilde{\sigma}_j^2$ might be larger than $\sigma_\textnormal{max}^2$, but is always bounded by $\Delta^2 \coloneqq \sup_{(i, j) \in E} \|U_{i:} - U_{j:}\|_2^2$.
The proof is concluded by observing that Cauchy-Schwarz inequality gives
\begin{align*}
\sum_{j \in S_\gamma(G)} \sqrt{1 + \tilde{\sigma}_j^2(|\mathcal{V}_j|-1)} \sqrt{\sum_{i \colon j(i) = i} T_i}  &\le \sqrt{1 + \Delta^2(N_\textnormal{max} - 1)}\,\sum_{j \in S_\gamma(G)} \sqrt{\sum_{i \colon j(i) = j} T_i}\\
&\le \sqrt{1 + \Delta^2(N_\textnormal{max} - 1)}\, \sqrt{\big|S_\gamma(G)\big| \sum_{j \in S_\gamma(G)} ~ \sum_{i \colon j(i) = j} T_i}\\
&= \sqrt{1 + \Delta^2(N_\textnormal{max} - 1)}\, \sqrt{\gamma(G) \,T}\,.
\end{align*}
\end{proof}


\subsection{Proof of Corollary \ref{cor:reggraphs_cliques}}

\corspecialgraph*

\begin{proof}
The first claim of \Cref{cor:reggraphs_cliques} is proved by using the last claim of \Cref{thm:any-G}, and recalling that for a $K$-regular graph we have $N_\textnormal{max} = N_\textnormal{min} = K + 1$.
Consider now a collection of $\chi$ cliques, $C_1, \ldots, C_{\chi}$.
Starting from the right-hand side of \Cref{eq:staring_point}, we have
\begin{align*}
\sum_{j=1}^N \frac{\sqrt{1 + \sigma_j^2(N_j-1)}}{\max_{i \in \scN_j} N_i} \sqrt{\sum_{i \in \scN_j} T_i} &= \sum_{k=1}^{\chi}\sum_{j\in C_k}\frac{\sqrt{1 + \sigma_j^2 (|C_k|-1)}}{|C_k| } \sqrt{\sum_{i \in C_k}T_i}\\
&= \sum_{k=1}^{\chi} \sqrt{1 + \sigma_\textnormal{max}^2 (|C_k|-1)} \sqrt{\sum_{i \in C_k}T_i}\\
&\le \sqrt{\sum_{k=1}^{\chi}\big(1 + \sigma_\textnormal{max}^2 (|C_k|-1)\big)} ~ \sqrt{\sum_{k=1}^{\chi} \sum_{i \in C_k}T_i}\\
&= \sqrt{\chi + \sigma_\textnormal{max}^2(N - \chi)}\, \sqrt{T}\,.
\end{align*}
\end{proof}


\subsection{Proof of Theorem \ref{thm:sto}}

\thmsto*

\begin{proof}
The proof follows that of Theorem \ref{thm:any-G} until the decomposition of  \Cref{eq:decompo_hedge}.
The analysis of \texttt{Hedge} with regularizers $\psi_t(\bm{p}) = \frac{\beta\sqrt{1+\sum_{s \le t-1}\Ind{i_s \in \scN_j}}}{\sqrt{\ln N_j}} \sum_{k=1}^{N_j} p_k \ln(p_k)$ then gives
\begin{align}
\sum_{t \colon i_t \in \mathcal{N}_j} \left\langle w_{i_tj}\,g_t,  \big[X_t\big]_{i_t:} - \big[X_t^{(\xi^*)}\big]_{i_t:}\right\rangle &= \sum_{t \colon i_t \in \mathcal{N}_j} \big\langle\,\text{loss}_t,  \bm{p}_t - e^*\big\rangle\nonumber\\
&\le \beta \sqrt{\ln N_j \sum_{i \in \scN_j}T_i} + \frac{\sqrt{\ln N_j}}{2\beta} \sum_{t\colon i_t \in \scN_j} \frac{\|\text{loss}_t\|_\infty^2}{\sqrt{1+\sum_{s\le t-1}\Ind{i_s \in \scN_j}}}\nonumber\\
&\le \beta \sqrt{\ln N_j \sum_{i \in \scN_j} T_i} + \frac{\sqrt{\ln N_j}}{2\beta} \sum_{t=1}^T \frac{w_{i_tj}^2\,\Ind{i_t \in \scN_j}}{\sqrt{1+\sum_{s \le t-1} \Ind{i_s \in \scN_j}}}\,.\label{eq:Ineedtorefertothis}
\end{align}
Taking expectation on both sides, we obtain
\begin{align*}
\mathbb{E}\left[\rule{0cm}{0.7cm}\right.\sum_{t \colon i_t \in \mathcal{N}_j} \Big\langle& w_{i_tj}\,g_t, \big[X_t\big]_{i_t:} - \big[X_t^{(\xi^*)}\big]_{i_t:}\Big\rangle\left.\rule{0cm}{0.7cm}\right]\\
&\le \mathbb{E}\left[\beta \sqrt{\ln N_j \sum_{i \in \scN_j}T_i} + \frac{\sqrt{\ln N_j}}{2\beta} \sum_{t=1}^T \frac{w_{i_tj}^2\,\Ind{i_t \in \scN_j}}{\sqrt{1+\sum_{s \le t-1} \Ind{i_s \in \scN_j}}}\right]\\
&\le \beta \sqrt{\ln N_j~\mathbb{E}\left[\sum_{i \in \scN_j}T_i\right]} + \frac{\sqrt{\ln N_j}}{2\beta} \mathbb{E}\left[\sum_{t=1}^T \mathbb{E}\left[\frac{w_{i_tj}^2\,\Ind{i_t \in \scN_j}}{\sqrt{1+\sum_{s \le t-1} \Ind{i_s \in \scN_j}}}\,\bigg|\, i_1, \ldots, i_{t-1}, i_t \in \scN_j\right]\right]\\
&\le \beta \sqrt{Q_j T\,\ln N_j} + \frac{\sqrt{\ln N_j}}{2\beta} \sum_{i \in \scN_j}\frac{q_i}{Q_j}\omega^2_{ij} ~ \mathbb{E}\left[\sum_{t=1}^T \frac{\Ind{i_t \in \scN_j}}{\sqrt{1+\sum_{s \le t-1} \Ind{i_s \in \scN_j}}}\right]\\
&\le \beta \sqrt{Q_j T\,\ln N_j} + \frac{\sqrt{\ln N_j}}{\beta} \sum_{i \in \scN_j}\frac{q_i}{Q_j}\omega^2_{ij} ~ \mathbb{E}\left[\sqrt{\sum_{i \in \scN_j} T_i}\right]\\
&\le 2 \sqrt{\sum_{i \in \scN_j} q_i\,w_{ij}^2} \sqrt{T\,\ln N_j}
\end{align*}
if we choose $\beta = \sqrt{\sum_{i \in \scN_j} \frac{q_i}{Q_j}w_{ij}^2}$.
Note that the same analysis can be applied to the second regret in the decomposition of \Cref{eq:decompo_hedge}.
Overall, we obtain that for all $U \in \mathcal{U}$ we have
\[
\mathbb{E}[R_T(U)] \le 6 \left(\sum_{j=1}^N \sqrt{\sum_{i \in \scN_j}q_i\,w_{ij}^2} \sqrt{1 + \sigma_j^2(N_j - 1)}\right) \sqrt{T\ln N}\,.
\]
To prove the second claim, substitute $\wij = q_j/Q_i$ in the above equation and observe that
\begin{align*}
\sum_{j=1}^N \sqrt{\sum_{i\in \scN_j} q_i\,w_{ij}^2} &= \sum_{j=1}^N \sqrt{\sum_{i\in \scN_j}\frac{q_i q_j^2}{Q_i^2}} = \sum_{j=1}^N q_j \sqrt{\sum_{i\in \scN_j}\frac{q_i}{Q_i^2}}\\
&\le \sqrt{\sum_{j=1}^N q_j \sum_{i\in \scN_j} \frac{q_i}{Q_i^2}} = \sqrt{\sum_{i=1}^N \sum_{j\in \scN_i} q_j \frac{q_i}{Q_i^2}} = \sqrt{\sum_{i=1}^N \frac{q_i}{Q_i}} \le \sqrt{\alpha(G)}\,,
\end{align*}
where the first inequality comes from Jensen's inequality, and the second one from a known combinatorial result, see e.g., \citet{griggs1983lower} or \citet[Lemma~3]{cesa2020cooperative}.
\end{proof}


\subsection{Extension of Theorem \ref{thm:sto} to unknown \texorpdfstring{$q_i$}{qi}}
\label{apx:unknown_q}

\begin{theorem}
Let $G$ be any graph, and assume that the agent activations are stochastic.
Consider the following strategy, run independently by each agent $j$ for its artificial clique, and based on its local time.
First, predict $Y_t^{(j)} = 0$ until the local time reaches $\tau \coloneqq \left\lceil \frac{4}{q_\textnormal{min}}\ln\big(2N^2T\big)\right\rceil$, where $q_\textnormal{min} = \min_{i \in [N]}q_i$.
Then, predict $Y_t^{(j)}$ using \MTFTRL, run with $N=N_j$ and $\beta_{t-1} = \sqrt{\sum_{i \in \scN_j} \hat{\pi}_{ij}(\tau)\,w_{ij}^2} \sqrt{1+\sum_{s \le t-1}\Ind{i_s \in \scN_j} - \tau}$, where the $\hat{\pi}_{ij}(\tau)$ are the empirical estimates of the $q_i/Q_j$ computed at local time $\tau$.
Then, the regret satisfies for all $U \in \mathcal{U}$
\[
\mathbb{E}[R_T(U)] \tildeO \hspace{-0.04cm}\left(\sum_{j=1}^N \sqrt{\sum_{i\in \mathcal{N}_j}q_i\,w_{ij}^2} \sqrt{1 + \sigma_j^2 (N_j-1)}\right) \hspace{-0.1cm}\sqrt{T} + \frac{N}{q_\textnormal{min}}\,,
\]
where the expectation is taken with respect to the agent activations, and $\tilde{\mathcal{O}}$ neglects logarithmic terms in $N$ and $T$.
\end{theorem}

\begin{proof}
We first introduce some notation.
For any $j \in [N]$, $i \in \scN_j$, and $t > 0$, let $t^{(j)} = \sum_{s \le t} \Ind{i_s \in \scN_j}$ be the local time for the artificial clique centered at agent $j$, $\pi_{ij} = q_i / Q_j$, and $\hat{\pi}_{ij}(t^{(j)}) = \sum_{s \le t} \Ind{i_s = i} / \sum_{s \le t} \Ind{i_s \in \scN_j} = \sum_{s \le t} \Ind{i_s = i} / t^{(j)}$ be its natural estimator at local time $t^{(j)}$.
The strategy is as follows.
First, each agent $j$ predicts with $Y_t^{(j)} = 0$ until its local time $t^{(j)}$ reaches some value $\tau_j$ to be determined later.
It then builds the estimates $\hat{\pi}_{ij}(\tau_j)$ of the $\pi_{ij}$.
Finally, from local time $\tau_j$ agent $j$ predicts $Y_t^{(j)}$ using its local instance of \MTFTRL, run with $N = N_j$, and $\beta_{t-1} = \sqrt{\sum_{i \in \scN_j} \hat{\pi}_{ij}(\tau_j)\,w_{ij}^2} \sqrt{1+\sum_{s \le t-1}\Ind{i_s \in \scN_j} - \tau_j}$.
We refer to this algorithm as \MTFTRLq.

%
%
%
Note that \Cref{lem:cool-cn} still applies, such that we may only focus on the regret on the different instances of \MTFTRLq run by the different agents.
We differentiate two cases, depending on whether the event
\[
\mathcal{E} \coloneqq \left\{\frac{\pi_{ij}}{2} \le \hat{\pi}_{ij}(\tau_j) \le \frac{3}{2}\pi_{ij}\,, \quad \forall j \in [N], \forall i \in \scN_j\right\}
\]
occurs or not.
If $\mathcal{E}$ is not satisfied, the regret of the instance of \MTFTRLq run by agent $j$ is trivially bounded by $T_j$, and that of the global approach by $T$.
When $\mathcal{E}$ is satisfied, we can control the regret of the instance of \MTFTRLq run by agent $j$ as follows.
On the first $\tau_j$ round, the regret is bounded by $\tau_j$.
From time step $\tau_j$ onward, we have $\frac{\pi_{ij}}{2} \le \hat{\pi}_{ij}(\tau_j) \le \frac{3}{2}\pi_{ij}$ and taking expectations on \eqref{eq:Ineedtorefertothis}, with $\beta = \sqrt{\sum_{i \in \scN_j} \hat{\pi}_{ij}(\tau_j)\,w_{ij}^2}$, we obtain, conditionally on $\mathcal{E}$
\begin{align*}
&\mathbb{E}\left[\rule{0cm}{1cm}\right.\sum_{\substack{t \colon t^{(j)} \ge \tau_j+1\\t \colon i_t \in \mathcal{N}_j}} \Big\langle w_{i_tj}\,g_t, \big[X_t\big]_{i_t:} - \big[X_t^{(\xi^*)}\big]_{i_t:}\Big\rangle ~\bigg|~ \mathcal{E}\left.\rule{0cm}{1cm}\right]\\
&\le \mathbb{E}\left[\rule{0cm}{1cm}\right. \sqrt{\sum_{i \in \scN_j} \hat{\pi}_{ij}(\tau_j)\,w_{ij}^2} \sqrt{\ln N_j \sum_{t \colon t^{(j)}\ge \tau_j+1}\Ind{i_t \in \scN_j}}  \\
&\hspace{1.2cm}+\frac{\sqrt{\ln N_j}}{2\sqrt{\sum_{i \in \scN_j} \hat{\pi}_{ij}(\tau_j)\,w_{ij}^2}} \sum_{t \colon t^{(j)}=\tau_j+1} \frac{w_{i_tj}^2\,\Ind{i_t \in \scN_j}}{\sqrt{1+\sum_{\substack{s \colon s^{(j)} \ge \tau_j+1\\ s \le t-1}} \Ind{i_s \in \scN_j}}}~\bigg|~ \mathcal{E}\left.\rule{0cm}{1cm}\right]\\
&\le \mathbb{E}\left[ \sqrt{\frac{3}{2}}\sqrt{\sum_{i \in \scN_j} \pi_{ij}\,w_{ij}^2} \sqrt{\ln N_j \sum_{t\ge1}\Ind{i_t \in \scN_j}} + \frac{\sqrt{2}\sqrt{\ln N_j}}{2\sqrt{\sum_{i \in \scN_j} \pi_{ij}\,w_{ij}^2}} \sum_{t=1}^T \frac{w_{i_tj}^2\,\Ind{i_t \in \scN_j}}{\sqrt{1+\sum_{s \le t-1} \Ind{i_s \in \scN_j}}}\right]\\
&\le 2\sqrt{2} \sqrt{\sum_{i \in \scN_j} q_i\,w_{ij}^2}\,\sqrt{T\,\ln N_j}\,.
\end{align*}

The same method applies to the regret with choice $\xi^*$, see decomposition \eqref{eq:decompo_hedge}, and we finally obtain that the expected regret of the overall approach, is bounded by
\[
\underbrace{\sum_{j=1}^N \tau_j}_{\coloneqq\tau} + \underbrace{9 \left(\sum_{j=1}^N \sqrt{\sum_{i \in \scN_j} q_i\,w_{ij}^2} \sqrt{1 + \sigma_j^2(N_j - 1)}\right) \sqrt{T\,\ln N}}_{\coloneqq\mathfrak{A}}\,,
\]
conditionally to $\mathcal{E}$.
Hence, we have
\begin{equation}
\mathbb{E}[R_T(U)] \le \mathbb{E}\big[\mathbb{P}(\mathcal{E})(\tau + \mathfrak{A}) + \mathbb{P}(\mathcal{E}^\mathsf{c})T\big] \le \tau + \mathfrak{A} + \mathbb{E}\big[\mathbb{P}(\mathcal{E}^\mathsf{c})\big]T\,.\label{eq:plug_proba}
\end{equation}
We now upper bound $\mathbb{P}(\mathcal{E}^\mathsf{c})$.
Let $q_\textnormal{min} = \min_{i \in [N]} q_i$.
For any $j \in [N]$ and $i \in \scN_j$, by the multiplicative Chernoff bound \citep[Corollary~4.6]{mitzenmacher2005probability} we have
\[
\mathbb{P}\left(\hat{\pi}_{ij}(\tau_j) \le \frac{\pi_{ij}}{2} \text{ or } \hat{\pi}_{ij}(\tau_j) \ge \frac{3}{2}\pi_{ij}\right) = \mathbb{P}\left(|\hat{\pi}_{ij}(\tau_j) -\pi_{ij}| \ge \frac{\pi_{ij}}{2}\right) \le \exp\left({-\frac{\pi_{ij}}{12}\,\tau_j}\right) \le \exp\left({-\frac{q_\textnormal{min}}{12}\,\tau_j}\right)\,,
\]
and by the union bound
\begin{equation}
\mathbb{P}(\mathcal{E}^\mathsf{c}) \le N \sum_{j=1}^N \exp\left(-\frac{q_\textnormal{min}}{12}\,\tau_j\right)\,.\label{eq:bound_exp_proba}
\end{equation}
Substituting \eqref{eq:bound_exp_proba} into \eqref{eq:plug_proba}, and setting $\tau_j = \left\lceil \frac{12}{q _\textnormal{min}}\ln\big(2N^2T\big)\right\rceil$ for all $j$, we get
\begin{align*}
\mathbb{E}[R_T(U)] &\le \mathfrak{A} + \tau + 2NT \sum_{j=1}^N \exp\left(-\frac{q_\textnormal{min}}{12}\, \tau_j\right)\\
&\le 9 \left(\sum_{j=1}^N \sqrt{\sum_{i \in \scN_j} q_i\,w_{ij}^2} \sqrt{1 + \sigma_j^2(N_j - 1)}\right) \sqrt{T\,\ln N} + \frac{12N}{q_\textnormal{min}}\ln\big(2N^2T\big) + N + 1\\
&\le 9 \left(\sum_{j=1}^N \sqrt{\sum_{i \in \scN_j} q_i\,w_{ij}^2} \sqrt{1 + \sigma_j^2(N_j - 1)}\right) \sqrt{T\,\ln N} + \frac{14N}{q_\textnormal{min}}\ln\big(2N^2T\big)\,.
\end{align*}
\end{proof}

\subsection{Proof of Theorem \ref{thm:lower-bound}}
\label{apx:lower}

\thmlowerbound*

\begin{proof}
\begin{figure*}[!t]
\centering
\begin{subfigure}{}
\caption{Activated nodes in a $2$-regular graph.}
\label{fig:2_regular}
\medskip
\includegraphics[width=0.8\textwidth]{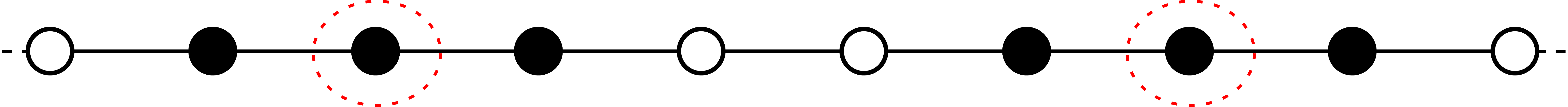}
\end{subfigure}
\bigskip
\begin{subfigure}{}
\caption{Activated nodes in a $4$-regular graph.}
\label{fig:4_regular}
\medskip
\includegraphics[width=0.8\textwidth]{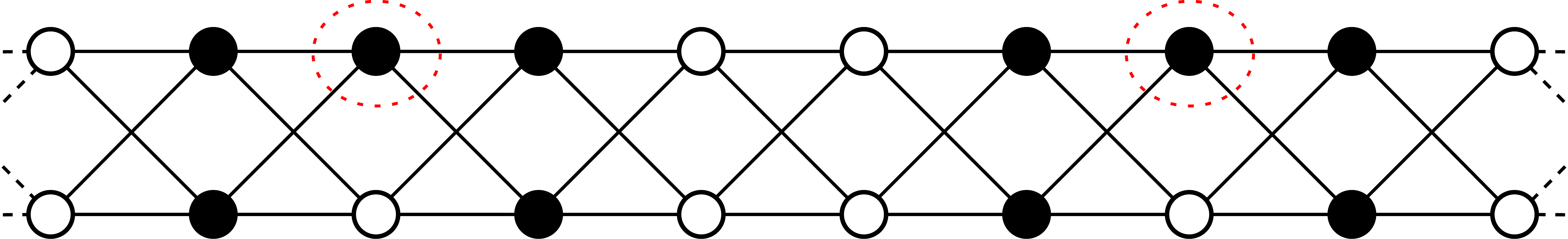}
\end{subfigure}
\bigskip
\begin{subfigure}{}
\caption{Activated nodes in a $6$-regular graph and more.}
\label{fig:6_regular}
\medskip
\includegraphics[width=0.8\textwidth]{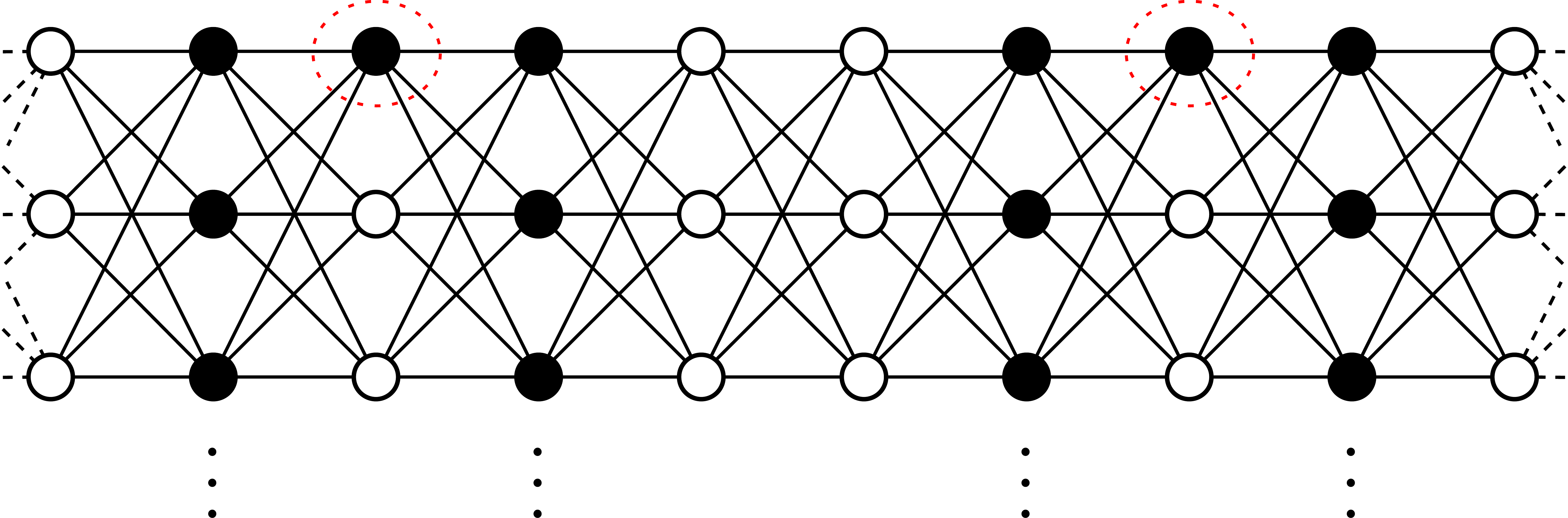}
\end{subfigure}
\bigskip
\end{figure*}

The first part of the lower bound is proved in \cite[Proposition~4]{cesa2022multitask}.
It establishes the existence of a comparator $U$ and a sequence of activations and gradients such that for any algorithm we have
\begin{equation}
R_T(U) \ge \frac{\sqrt{2}}{4}\sqrt{1 + \sigma^2(N-1)}\sqrt{T}\,.\label{eq:lb_variance}  
\end{equation}
Note that this lower bound is oblivious to the communication graph $G$, and applies in particular to algorithms that follow the protocol described in \Cref{sec:model} (where information can only be exchanged along the edges of  $G$).
For the second part of the first lower bound, let $S_{\alpha_2}(G) = \{j_1, \ldots, j_{\alpha_2}\}$ be a largest twice independent set of $G$.
Due to the communication protocol, it is immediate to check that if only nodes in $S_{\alpha_2}(G)$ are activated, then no information from an active node can be transmitted to another active node.
Hence, the regret incurred by the nodes in $S_{\alpha_2}(G)$ are independent.
Now, we know from standard online learning lower bounds that there exists some $u_0 \in \mathbb{R}^d$ and a sequence $g_1, \ldots, g_{T}$ such that the regret of a single agent is lower bounded by $\sqrt{T}/2$, see e.g., \cite[Theorem~5.1]{orabona2019modern}.
Let $U = (u_0, \ldots, u_0) \in \mathbb{R}^{N \times d}$, the activations be $j_1$ for the first $T/\alpha_2(G)$ steps, $j_2$ for the next $T/\alpha_2(G)$, and so on.
Finally, consider the sequence of gradients $g_1 \ldots g_{T/\alpha_2(G)}$ repeated $\alpha_2(G)$ times.
In words, each agent in $S_{\alpha_2}(G)$ is sequentially given $g_1 \ldots g_{T/\alpha_2(G)}$ and comparator $u_0$.
Then the multitask regret satisfies
\begin{equation}
R_T(U) = \sum_{j \in S_{\alpha_2}(G)} R_{T/\alpha_2(G)}^\textnormal{agent $j$} \ge \sum_{j \in S_{\alpha_2}(G)} \frac{\sqrt{T/\alpha_2(G)}}{2} = \frac{\sqrt{\alpha_2(G)\,T}}{2}\,,\label{eq:lb_twice_indep}
\end{equation}
where the first equality comes from the independence due to the communication constraints, and the inequality from the single-task lower bound.
Gathering \eqref{eq:lb_variance} and \eqref{eq:lb_twice_indep} yields the desired result.

Consider now the special case of $K$-regular graphs.
\Cref{fig:2_regular,fig:4_regular,fig:6_regular} show examples of communication graph on which our second lower bound can be attained.
The idea is similar to that of the first lower bound.
Only groups of nodes that are $2$-nodes apart (the black nodes in \Cref{fig:2_regular,fig:4_regular,fig:6_regular}) are activated.
This way, the regret incurred by the different groups are summed, as they are independent.
On a group of nodes, one can identify a ``central node'' (circled with red dashes).
We denote the set of indices of such  nodes $\mathcal{J}_\textnormal{cn}$.
By the multitask lower bound of \citet{cesa2022multitask}, we know that for the (artificial) clique with center node $j$ there exists a comparator $U^{(j)} \in \mathbb{R}^{N_j \times d}$ and a sequence of activations and gradients such that we have a lower bound of order $\sqrt{1 + K \, \sigma_j^2} \sqrt{\sum_{i \in \scN_j} T_i}$.
Overall by feeding these gradients to each groups of nodes, each activated $T/|\mathcal{J}_\textnormal{cn}|$ times, we obtain
\[
R_T(U) = \sum_{j \in \mathcal{J}_\textnormal{cn}} R_{T/|\mathcal{J}_\textnormal{cn}|}^\textnormal{group centered at $j$}(U^{(j)}) \ge \frac{\sqrt{2}}{4}\sum_{j \in \mathcal{J}_\textnormal{cn}} \sqrt{1 + K\,\sigma_j^2} \, \sqrt{\frac{T}{|\mathcal{J}_\textnormal{cn}|}} \ge \frac{1}{5} \sqrt{1 + K\,\sigma_\textnormal{min}^2} \, \sqrt{\frac{NT}{K}}\,,
\]
where we have used that on the family of regular graphs depicted in \Cref{fig:2_regular,fig:4_regular,fig:6_regular}, we have $|\mathcal{J}_\textnormal{cn}| = \frac{2N}{5K}$.
\end{proof}

\subsection{A Failing Approach}
\label{sec:fail}

In \mtcool, each agent maintains an instance of \MTFTRL for the artificial clique composed of its neighbors.
However, this is not the only way that one could use \MTFTRL as a building block to devise a generic algorithm operating on any communication graph.
Another tempting approach consists in maintaining one instance of \MTFTRL for every pair of agents $(i, j)$ such that $(i, j) \in E$.
Let $Y^{(i, j)}_t \in \mathbb{R}^{2 \times d}$ be the prediction maintained at time step $t$ by such an algorithm.\footnote{We consider here for simplicity that every agent is connected to itself, such that we also maintain $Y_t^{(i, i)}$ for all $i \in V$. The latter is exclusively based on the feedback received by agent $i$.}
Note that we have $Y^{(i, j)}_t = Y^{(j, i)}_t$ for any $(i, j) \in E$.
Similarly to \mtcool, a natural way for the active agent $i_t$ to make predictions is to submit the weighted average $x_t = \sum_{j \in \scN_{i_t}} w_{i_tj} \big[Y_t^{(i_t, j)}\big]_{i_t:}$, where the $w_{ij}$ are nonnegative and such that $\sum_{j \in \scN_{i_t}} w_{ij} = 1$ for all $i$.
Note that, contrary to \mtcool, predictions do not require a fetch step here, as both endpoints of each edge $(i, j)$ are maintaining the same $Y^{(i,j)}_t = Y^{(j,i)}_t$.
Using similar arguments as for \mtcool, we obtain
\begin{align}
\sum_{t=1}^T \ell_t(x_t) - \ell_t(U_{i_t:}) &\le \sum_{t=1}^T \langle g_t, x_t - U_{i_t:}\rangle\nonumber\\
&= \sum_{t=1}^T \sum_{i=1}^N  \Big\langle g_t, \sum_{j \in \mathcal{N}_i} \wij\big[Y_t^{(i,j)}\big]_{i:} -U_{i:} \Big\rangle \Ind{i_t=i}\nonumber\\
&= \sum_{t=1}^T \sum_{i=1}^N \sum_{j \in \mathcal{N}_i} \left\langle g_t,   \wij\left(\big[Y_t^{(i,j)}\big]_{i:} -U_{i:} \right) \right\rangle \Ind{i_t=i}\nonumber\\
&= \sum_{i=1}^N \sum_{j \in \mathcal{N}_i} \sum_{t=1}^T \left\langle \wij\,g_t,   \left(\big[Y_t^{(i,j)}\big]_{i:} -U_{i:} \right) \right\rangle \Ind{i_t=i}\label{eq:i_active}\\
&= \sum_{j=1}^N \sum_{i \in \mathcal{N}_j} \sum_{t=1}^T \left\langle \wij\,g_t,   \left(\big[Y_t^{(i,j)}\big]_{i:} -U_{i:} \right) \right\rangle \Ind{i_t=i}\label{eq:invert_i_active}\\
&= \sum_{i=1}^N \sum_{j \in \mathcal{N}_i} \sum_{t=1}^T \left\langle w_{ji}\,g_t, \left(\big[Y_t^{(i,j)}\big]_{j:} -U_{j:} \right) \right\rangle \Ind{i_t=j}\label{eq:j_active}\\
&= \frac{1}{2} \sum_{i=1}^N \sum_{j \in \mathcal{N}_i} \sum_{t=1}^T \bigg[\left\langle \wij\,g_t, \left(\big[Y_t^{(i,j)}\big]_{i:} -U_{i:} \right) \right\rangle \Ind{i_t=i}\nonumber\\
&\hspace{2.7cm}+ \left\langle w_{ji}\,g_t, \left(\big[Y_t^{(i,j)}\big]_{j:} -U_{j:} \right) \right\rangle\Ind{i_t=j}\bigg]\label{eq:ij_active}\\
&= \sum_{(i,j) \in E} R^{(i,j)}_T\big(U^{(i,j)}\big)\,,\nonumber
\end{align}
where \eqref{eq:invert_i_active} derives from \Cref{lem:invert}, \eqref{eq:j_active} is \eqref{eq:invert_i_active} with indices $i$ and $j$ swapped, \eqref{eq:ij_active} is the average of \eqref{eq:i_active} and~\eqref{eq:j_active}, and $R^{(i,j)}_T\big(U^{(i,j)}\big)$ denotes the regret suffered by the instance of \MTFTRL run by edge $(i,j)$ on the linear losses $\big\langle (w_{ij}\Ind{i_t=i}+ w_{ji}\Ind{i_t=j})\,g_t, \cdot\big\rangle$ over the rounds $t \le T$ such that $i_t=i$ or $i_t=j$.
Now, by the \MTFTRL analysis, assuming stochastic activations we have
\[
R^{(i,j)}_T\big(U^{(i,j)}\big) \tildeO \sqrt{(q_i w_{ij}^2 + q_j w_{ji}^2)(1 + \Delta_{ij}^2)}\sqrt{T}\,,
\]
where $\Delta_{ij}^2 = \|U_{i:} - U_{j:}\|_2^2$.
Hence, if we assume uniform activations (i.e., $q_i = 1/N$ for all $i$) in the single-task case (i.e., $\Delta_{ij}^2 = 0$ for all $(i, j) \in E$) we obtain overall
\[
R_T(U) \tildeO \sum_{(i, j) \in E} \sqrt{w_{ij}^2 + w_{ji}^2}\sqrt{T/N}\,,
\]
which can be further simplified into
\begin{equation}
\sum_{(i, j) \in E} \sqrt{w_{ij}^2 + w_{ji}^2}\sqrt{T/N} \le \sum_{(i, j) \in E} \big(w_{ij} + w_{ji}\big)\sqrt{T/N} = \sqrt{NT}\,.\label{eq:bad_regret}
\end{equation}
Contrary to \mtcool, this strategy thus fails to improve over the naive bound using independent updates (\texttt{iFTRL}).
Note that the inequality in \eqref{eq:bad_regret} is tight, while the equality holds for any $w_{ij}$.
Hence, \textbf{there is no choice of weights that allows achieving better performance than independent runs of \FTRL}.
This can be intuitively explained as this strategy runs \MTFTRL on small subsets (pairs), which makes learning slower as they are activated less often.
Instead, \mtcool leverages \MTFTRL on the largest possible set of nodes that can communicate with the central agent, i.e., its neighbourhood.


\section{Technical Proofs (Results from Section \ref{sec:DP})}

In this section, we gather the technical proofs of the results exposed in \Cref{sec:DP}.
First, we provide some intuition about the construction of \DOPE.
Then, we prove that \DOPE is $\epsilon$-DP (\Cref{th:privacy}) before bounding its regret (\Cref{th:reg-DP_anyG}).

\subsection{Intuition on \texorpdfstring{\DOPE}{dope} and Explicitation of \texorpdfstring{\texttt{DPMT-FTRL}}{DPMT-FTRL}}

In this section, we provide more intuition about the construction of \DOPE.
In particular, we explain why \DOPE invokes \DPMTFTRL, instead of any multitask algorithm working on a clique like \mtcool.
\DPMTFTRL is a variant of \MTFTRL that works with a different kind of feedback. 
Instead of observing the gradient of the loss at the prediction, \DPMTFTRL observes sanitized versions of cumulative sums that are needed to compute the expert predictions $X_t^{(\xi)}$ and their weights $p_t^{(\xi)}$.
To  explain this difference, we dissect \mtcool and highlight which changes are necessary to make it DP.

Extending the notation of \Cref{alg:mt-ftrl,alg:mt-ftrl-DP}, we denote by $X_t^{(j, \xi)}$ and $p_t^{(j, \xi)}$ the experts and probabilities maintained by agent $j$.
Let $\mathcal{U}_j \coloneqq \big\{U \in \mathbb{R}^{N_j \times d} \colon U_{i:} \in \mathcal{X} \text{ for all }i \le N_j\big\}$ and $\Delta_j$ the simplex in $\mathbb{R}^{N_j}$.
In \mtcool, we have
\begin{align}
X_t^{(j, \xi)} &= \argmin_{\substack{X \in \mathcal{U}_j\\\sigma^2(X) \le \xi}} ~~ \eta_{t-1}^{(\xi)} \sum_{s \le t-1} \big\langle w_{i_sj}\,g_s, X_{i_s:}\big\rangle \, \Ind{i_s \in \scN_j} + \frac{1}{2}\|X\|_{A_j}^2\nonumber\\
&= \argmin_{\substack{X \in \mathcal{U}_j\\\sigma^2(X) \le \xi}} ~~ \eta_{t-1}^{(\xi)} \sum_{s \le t-1} \, \sum_{i \in \scN_j} \big\langle w_{ij}\, g_s, X_{i:}\big\rangle\,\Ind{i_s=i} + \frac{1}{2}\|X\|_{A_j}^2\nonumber\\
&= \argmin_{\substack{X \in \mathcal{U}_j\\\sigma^2(X) \le \xi}} ~~ \eta_{t-1}^{(\xi)} \sum_{i \in \scN_j} \left\langle w_{ij}\underbrace{\sum_{s \le t-1} g_s \, \Ind{i_s=i}}_{\coloneqq \gamma_t^{(i)}}, X_{i:}\right\rangle + \frac{1}{2}\|X\|_{A_j}^2\,.\label{eq:expert}
\end{align}
Hence, agent $j$ must receive private versions of the $\gamma_t^{(i)}$ from each of its neighbor $i \in \scN_j$, which are precisely the $\tilde{\gamma}_t^{(i)}$ in \Cref{alg:DOPE_learning}.
Note that the $\gamma_t^{(i)}$ are sums of gradients, that can be sanitized efficiently using \treeBasedAgg \citep{chan2011private}.
But the $\gamma_t^{(i)}$ are not the only gradient information agent $j$ needs from its neighbors.
Indeed, we have
\begin{align}
p_t^{(j, :)} &= \argmin_{p \in \Delta_j} ~~ \frac{\sqrt{\ln N_j}}{\beta_{t-1}} \sum_{s \le t-1} \langle \text{loss}_s, p\rangle \, \Ind{i_s \in \scN_j} + \sum_{\xi \in \Xi_j} p^{(\xi)} \ln p^{(\xi)}\nonumber\\
&= \argmin_{p \in \Delta_j} ~~ \frac{\sqrt{\ln N_j}}{\beta_{t-1}} \sum_{s \le t-1} \sum_{\xi \in \Xi_j} p^{(\xi)} \Big\langle w_{i_sj}\,g_s, \big[X_s^{(j, \xi)}\big]_{i_s:}\Big\rangle \Ind{i_s \in \scN_j} + \sum_{\xi \in \Xi_j} p^{(\xi)} \ln p^{(\xi)}\nonumber\\
&= \argmin_{p \in \Delta_j} ~~ \frac{\sqrt{\ln N_j}}{\beta_{t-1}} \sum_{\xi \in \Xi_j} p^{(\xi)} \sum_{i \in \scN_j} \underbrace{w_{ij}\sum_{s \le t-1} \Big\langle g_s, \big[X_s^{(j, \xi)}\big]_{i:}\Big\rangle \Ind{i_s = i}}_{\coloneqq s_{t, i}^{(j, \xi)}} + \sum_{\xi \in \Xi_j} p^{(\xi)} \ln p^{(\xi)}\,.\label{eq:proba}
\end{align}
Hence, agent $j$ also needs to receive private versions of the $s_{t, i}^{(j, \xi)}$, which are the $\tilde{s}_{t, i}^{(j, \xi)}$ in \Cref{alg:DOPE_learning}.
Note that, again, the $s_{t, i}^{(j, \xi)}$ are sums, so that they can be sanitized easily using \treeBasedAgg.

The tree aggregations needed to sanitize the $\gamma_t^{(i)}$ and the $s_{t, i}^{(j, \xi)}$ are taken care of by instructions contained in \Cref{alg:DOPE_learning}, while the modifications of \MTFTRL needed to take private cumulative sums as inputs are outlined in \Cref{alg:mt-ftrl-DP}. 
In the latter pseudo-code, we drop the superscripts $^{(j)}$ since we consider \DPMTFTRL on a clique only. 



\begin{algorithm}[t]
\caption{~\texttt{DPMT-FTRL} (on linear losses)}\label{alg:mt-ftrl-DP}
\vspace{0.05cm}

\Req{Number of agents $N$, learning rates $\beta_{t-1}$}
\vspace{0.05cm}

\Init{$A = (1+N) I_N - \bm{1}_N\bm{1}_N^\top$,~~$\Xi = \{1/N,\,2/N, \ldots, 1\}$,~~$p_1^{(\xi)} = \frac{1}{N}$~$\forall\,\xi \in \Xi $, $\hatGammaIn_0 = 0_{\mathbb{R}^{N\times d}}$, $\tilde{s}_{t}^{(\xi)}= 0_{\mathbb{R}^{N}}$~$\forall\,\xi \in \Xi$}
\vspace{0.05cm}

\For{$t = 1, 2, \ldots$}{\vspace{0.05cm}

\For{$\xi \in \Xi$}{\vspace{0.05cm}
{\tcp{\small \textcolor{blue}{Set learning rate assuming $\sigma^2 = \xi$}}}
\vspace{0.05cm}

$\eta^{(\xi)}_{t-1} = \frac{N}{\beta_{t-1}}\sqrt{1+\xi(N-1)}$\vspace{0.15cm}

\tcp{\small \textcolor{blue}{FTRL with Mahalanobis regularizer}}\vspace{0.05cm}

$\displaystyle X_{t}^{(\xi)} = \argmin_{\substack{X \in \mathcal{U}\\\sigma^2(X) \le \xi}} ~ \eta^{(\xi)}_{t-1} \left\langle \hatGammaIn_{t-1} , X\right\rangle + \frac{1}{2} \|X\|_A^2$
\vspace{0.15cm}
}
\tcp{Predict and receive feedback}
\vspace{0.05cm}

Predict $Y_t  = \sum_{\xi \in\,\Xi} ~ p_t^{(\xi)}\,X_{t}^{(\xi)}$
\vspace{0.15cm}

Incur loss $\langle g_t, [Y_t]_{i_t:} \rangle$ and receive\vspace{0.05cm}

\qquad $\tilde{\gamma}_t^{(i_t)}$ the sanitized version of $\gamma_t^{(i_t)} = \sum_{s\le t} g_s\,\Ind{i_s = i_t}$
\vspace{0.07cm}

\qquad $\tilde{s}_{t, i_t}^{(\xi)}$ the sanitized version of $s_{t, i_t}^{(\xi)} = \sum_{s \le t-1} \big\langle g_s, \big[X_s^{(\xi)}\big]_{i_t:}\big\rangle\,\Ind{i_s = i_t}$, for all $\xi \in \Xi$
\vspace{0.15cm}
 
\tcp{Update $\hatGammaIn_t$ and $\tilde{s}_t^{(\xi)}$}\vspace{0.05cm}
$[\hatGammaIn_t ]_{i_t:} = \tilde{\gamma}_t^{(i_t)} \text{ and }[\hatGammaIn_t ]_{i:} = [\hatGammaIn_{t-1} ]_{i:}~\,\forall i \neq i_t$\vspace{0.15cm}

\For{$\xi \in \Xi$}{\vspace{0.05cm}
$\tilde{s}_{t,i_t}^{(\xi)} = \tilde{s}_{t,i_t}^{(\xi)} \text{ and } \tilde{s}_{t,i}^{(\xi)}=\tilde{s}_{t-1,i}^{(\xi)}~\,\forall i \neq i_t$\vspace{0.15cm}
}


\tcp{Update $p_t$ based on the experts losses}
\vspace{0.05cm}

\For{$\xi \in \Xi$}{

$\displaystyle p_{t+1}^{(\xi)} = \frac{\exp\left(-\frac{\sqrt{\ln N}}{\beta_t}\sum_{i=1}^N \tilde{s}_{t,i}^{(\xi)}\right)}{\sum_k\,\exp\left(-\frac{\sqrt{\ln N}}{\beta_t} \sum_{i=1}^N  \tilde{s}_{t,i}^{(k)} \right)}$
}
}
\end{algorithm}

\paragraph{Information flow.}
The above discussion allows understanding which information needs to be exchanged between neighbors.
In particular, the messages exchanged between agents, even if sanitized, carry information about the gradient sequence.
Consequently,  analyzing the content of these messages is crucial for the analysis in terms of DP of \DOPE. 
Overall, the message $m_t^{(i \shortrightarrow j)}$ sent by agent $i$ to agent $j$ at iteration $t$ of \DOPE is
\begin{equation}
m_t^{(i \shortrightarrow j)}= \begin{cases}
p_t^{(i, :)}\,,\, \Big\{\big[X_{t}^{(i, \xi)}\big]_{j:} \colon \xi \in \Xi_i\Big\} & \text { if } i \in \mathcal{N}_{i_t} \setminus \{i_t\} \text{ and } j=i_t \text{ (fetch step)}\\[0.2cm]
w_{ij}\,\tilde{\gamma}_t^{(i)}\,,\,\left\{\tilde{s}_{t, i}^{(j, \xi)} \colon \xi \in \Xi_j\right\} & \text { if } i=i_t \text{ and } j \in \scN_{i_t} \text{ (send step)}\\[0.2cm]
\,\emptyset &\text { otherwise}\\[0.1cm]
\end{cases}\label{eq:message}
\end{equation}
Note that sending $p_t^{(i, :)}$ and $\Big\{\big[X_{t}^{(i, \xi)}\big]_{j:} \colon \xi \in \Xi_i\Big\}$ is actually equivalent (information and privacy-wise) to sending $\big[Y^{(i)}_t\big]_{i_t:}$ as written in \Cref{alg:DOPE_learning}, since agent $i_t$ can compute $\big[Y_t^{(i)}\big]_{i_t:} = \sum_{\xi \in \Xi_i} p_t^{(i, \xi)} \big[X_t^{(i, \xi)}\big]_{i_t:}$.
In addition, having access to the individual experts $\big[X_{t}^{(i, \xi)}\big]_{j:}$ is necessary for $i_t$ to compute the $\tilde{s}_{t, i_t}^{(i, \xi)}$.
In what follows we denote $m^{(i)}_t = \big\{m_t^{(i \shortrightarrow j)} \colon j \in \scN_i\big\}$ the batch of messages sent by agent $i$ to other agents at time step $t$.

\subsection{Proof of Theorem \ref{th:privacy}}

\thmprivacy*

\begin{proof}
Recall from \Cref{def:loss-level-priv} that we should check that for any $i \in [N]$ and set of message sequences $\mathcal{M}$ we have
\[
\frac{\sP\big(m^{(i)}_{1}, \ldots, m^{(i)}_{T} \in \mathcal{M} \mid i_1, g_1, \ldots, i_T, g_T\big)} {\sP\big(m^{(i)}_{1}, \ldots, m^{(i)}_{T}  \in \mathcal{M} \mid i_1, g'_1, \ldots, i_T, g'_T\big)} \leq e^{\epsilon}\,.
\]
We will show that the above equation holds by focusing on each possible component of message $m_t^{(i)}$, see \eqref{eq:message}.
Let $\tau$ be the round where sequences of gradients $(g_t)_{t \in [T]}$ and $(g_t')_{t \in [T]}$ differ, we start by identifying which messages are impacted by the change of $g_\tau$ into $g'_\tau$:
\begin{itemize}[topsep=0pt]
\item By definition, $\tilde{\gamma}_t^{(i)}$ is impacted iff $i = i_\tau$.
\item Recalling \eqref{eq:expert}, $X_t^{(i, \xi)}$ is impacted iff $i \in \scN_{i_\tau}$.
\item By definition, $\tilde{s}_{t, i}^{(j, \xi)}$ is impacted iff $\sum_t g_t\,\Ind{i_t = i}$ or the $X_t^{(j, \xi)}$ are impacted, i.e., iff $i=i_t$ or $j \in \scN_{i_\tau}$.
\item Recalling \eqref{eq:proba}, $p_t^{(i, :)}$ is impacted iff any $\tilde{s}_{t, j}^{(i, \xi)}$ for $j \in \scN_i$ is impacted, i.e., iff $i \in \scN_{i_\tau}$.
\end{itemize}

We now quantify the ratios of probabilities that the above elements of messages belong to $\mathcal{M}$, given the different sequences of gradients.
Note that the latter elements may not belong to the same space (e.g., $p_t^{(i, :)} \in \Delta_i$ while $\tilde{\gamma}_t^{(i)} \in \mathbb{R}^d$), such that in the following we use $\mathcal{M}$ as a generic set that adapts to the type of element considered.
We highlight that for any $i, j$ such that the message element is not impacted, the probability ratio is equal to $1$, such that in particular it is bounded by the terms exhibited at \Cref{eq:ratio_proba_gamma,eq:ratio_proba_X,eq:ratio_proba_s,eq:ratio_proba_p}.

Recall that variables $\tilde{\gamma}_t^{(i)}$ are generated by a tree aggregation with at most $T$ summands, that have a maximal $L_1$ norm of $\sqrt{d} \max_{i \in \scN_j} w_{ij}\leq \sqrt{d}$, and a noise following a Laplacian distribution with parameter $\sqrt{d}\frac{\ln T}{\epsilon'}$.
Hence by \citet[Theorem~3.5]{chan2011private} \ifthenelse{\boolean{long}}{and the decomposition introduced in  \citet[Equation ~(11)]{smith2017federated}, allowing to deal with the dependence on a given gradient at some round of subsequent gradients, }{}for any $i \in [N]$ we have
\begin{equation}
\frac{\sP\big(\tilde{\gamma}^{(i)}_{1}, \ldots, \tilde{\gamma}^{(i)}_{T} \in \mathcal{M} \mid i_1, g_1, \ldots, i_T, g_T\big)} {\sP\big(\tilde{\gamma}^{(i)}_{1}, \ldots,\tilde{\gamma}^{(i)}_{T} \in \mathcal{M} \mid i_1, g'_1, \ldots, i_T, g'_T\big)} \leq \exp{(\epsilon')}\,.\label{eq:ratio_proba_gamma}
\end{equation}
We remind that the noise completion introduced by \cite{agarwal2017price} does not harm the privacy, as it can be considered as post-processing.
%
\medskip

Next, let $X_t^{(i, :)} = \big\{X_{t}^{(i, \xi)} \colon \xi \in \Xi_i\big\}$.
By the construction of the $X_t^{(i, \xi)}$, which are deterministic functions of the $\gamma_t^{(j)}$ for $j \in \scN_i$, see \eqref{eq:expert}, we have for any $i \in [N]$
\begin{equation}
\frac{\sP\big(X^{(i, :)}_{1}, \ldots, X^{(i, :)}_{T} \in \mathcal{M} \mid i_1, g_1, \ldots, i_T, g_T\big)}{\sP\big(X^{(i, :)}_{1}, \ldots, X^{(i, :)}_{T} \in \mathcal{M} \mid i_1, g'_1, \ldots, i_T, g'_T\big)} \leq \exp{(\epsilon')}\,.\label{eq:ratio_proba_X}
\end{equation}
\medskip

We now focus on the $\tilde{s}_{t, i}^{(j, \xi)}$, and denote $\tilde{s}_{t, i}^{(j, :)} = \big\{\tilde{s}_{t, i}^{(j, \xi)} \colon \xi \in \Xi_j\big\}$.
For any $i, j$ we have
\begin{align}
&\frac{\sP\Big(\tilde{s}_{1, i}^{(j, :)}, \ldots, \tilde{s}_{T, i}^{(j, :)} \in \mathcal{M} \mid i_1, g_1, \ldots, i_T, g_T\Big)} {\sP\Big(\tilde{s}_{1, i}^{(j, :)}, \ldots,\tilde{s}_{T, i}^{(j, :)} \in \mathcal{M} \mid i_1, g'_1, \ldots, i_T, g'_T\Big)}\nonumber\\
&= \bigintsss_{\mathcal{X}}{\frac{\sP\Big(\tilde{s}_{1, i}^{(j, :)}, \ldots, \tilde{s}_{T, i}^{(j, :)} \in \mathcal{M} \mid i_1, g_1,  \ldots, i_T, g_T, \big(X^{(j, :)}_{s}\big)_{s\in [T]} \in \partial \mu \Big)} {\sP\Big(\tilde{s}_{1, i}^{(j, :)}, \ldots, \tilde{s}_{T, i}^{(j, :)} \in \mathcal{M} \mid i_1, g'_1, \ldots, i_T, g'_T, \big(X^{(j, :)}_{s}\big)_{s\in [T]} \in \partial \mu \Big)}  \frac{\sP\Big(\big(X^{(j, :)}_{s}\big)_{s\in [T]} \in \partial \mu \mid i_1, g_1, \ldots, i_T, g_T\Big)} {\sP\Big(\big(X^{(j, :)}_{s}\big)_{s\in [T]} \in \partial \mu  \mid i_1, g'_1, \ldots, i_T, g'_T\Big)} \partial \mu} \nonumber\\
&\le {\prod}_{\xi \in \Xi_j}{ \bigintsss_{\mathcal{X}'}{\frac{\sP\Big(\tilde{s}_{1, i}^{(j, \xi)}, \ldots, \tilde{s}_{T, i}^{(j, \xi)} \in \mathcal{M} \mid i_1, g_1,  \ldots, i_T, g_T, \big(X^{(j, \xi)}_{s}\big)_{s\in [T]} \in \partial \mu' \Big)} {\sP\Big(\tilde{s}_{1, i}^{(j, \xi)}, \ldots, \tilde{s}_{T, i}^{(j, \xi)} \in \mathcal{M} \mid i_1, g'_1, \ldots, i_T, g'_T, \big(X^{(j, \xi)}_{s}\big)_{s\in [T]} \in \partial \mu' \Big)} \times \exp(\epsilon')\, \partial \mu'}}\label{eq:split_indep}\\
&\le \exp\big((N_j+1)\epsilon'\big)\label{eq:ratio_proba_s}\\
&\le \exp\big(2N_j\epsilon'\big)\,,\nonumber
\end{align}
where \eqref{eq:split_indep} derives from \eqref{eq:ratio_proba_X}, and \eqref{eq:ratio_proba_s} from \citet[Theorem~3.5]{chan2011private} and the decomposition introduced in  \citet[Equation ~(11)]{smith2017federated}, since each $\tilde{s}_{t, i}^{(j, \xi)}$ is generated by a tree aggregation with at most $T$ inner products between the predictions of the experts and the gradients (which in particular have a $L_1$ norm bounded by $\max_{j  \in \mathcal{N}_i}w_{ij}\le 1 $), and a noise following a Laplacian distribution with parameter $\ln T/ \epsilon'$. $\mathcal X$ denotes the space of sequences of real matrices of dimension $N_j \times N_j$ $\mathcal X'$ denotes the space of sequences of vectors of dimension $N_j$.
\medskip

Finally, noticing that $p_t^{(i, :)}$ is just a post-processing of the $\tilde{s}_{t, j}^{(i, \xi)}$ for $j \in \scN_i$ and $\xi \in \Xi_i$, see \eqref{eq:proba}, we have for any $i \in [N]$
\begin{equation}
\frac{\sP\big(p_1^{(i, :)}, \ldots, p_T^{(i, :)} \in \mathcal{M} \mid i_1, g_1, \ldots, i_T, g_T\big)}{\sP\big(p_1^{(i, :)}, \ldots, p_T^{(i, :)} \in \mathcal{M} \mid i_1, g'_1, \ldots, i_T, g'_T\big)} \leq \exp\big(N_i(N_i+1)\epsilon'\big) \le \exp\big(2N_i^2\epsilon'\big)\,.\label{eq:ratio_proba_p}
\end{equation}
\medskip

Overall, $\big(m_t^{(i)}\big)_{t \in [T]}$ is a post-processing of: $\big(\tilde{\gamma}_t^{(i)}\big)_{t \in [T]}$, $\big(X_t^{(i, :)}\big)_{t \in [T]}$, $\big(\tilde{s}_{t,i}^{(j, :)}\big)_{t \in [T], j \in \scN_i}$, $\big(p_t^{(i, :)}\big)_{t \in [T]}$, such that for all $i \in [N]$ we have
\[
\frac{\sP\big(m_1^{(i)}, \ldots, m_T^{(i)} \in \mathcal{M} \mid i_1, g_1, \ldots, i_T, g_T\big)}{\sP\big(m_1^{(i)}, \ldots, m_T^{(i)} \in \mathcal{M} \mid i_1, g'_1, \ldots, i_T, g'_T\big)} \leq \exp\left(\epsilon' + \epsilon' + 2N_i^2\epsilon' + 2N_i^2\epsilon'\right) \le \exp\big(6N_\textnormal{max}^2\epsilon'\big) = \exp(\epsilon)\,,
\]
where we have used that $\epsilon' = \epsilon/(6N_\textnormal{max}^2)$.

\end{proof}

\subsection{Proof of Theorem \ref{th:reg-DP_anyG}}

We first recall two results, which are instrumental to our analysis.

\begin{lemma}[An alternative version of Theorem 3.4 of \cite{agarwal2017price}]
\label{thm:mainftrlthm2}
For any noise distribution $\mathcal{D}$, regularizers $\psi_t(x)= \frac{\psi(x)}{\eta_{t-1}}$ with non decreasing $(\eta_t)_{t \ge 1}$ and $\psi$ that is $\mu$-strongly convex with respect to $\|\cdot\|$,  decision set $\mathcal{X}$, and loss vectors $g_1, \ldots g_T$, the regret of \FTRL on sums of gradients sanitized through \treeBasedAgg set with distribution $\mathcal{D}$ is bounded by
\[
\E[\textnormal{Regret}_T(u)] \leq \frac{\psi(u)- \min_{x \in \mathcal{X}} \psi(x)}{\eta_{T}} +\frac{1}{2\mu}\sum_{t = 1}^{T} \eta_{t-1} \|g_t\|_*^2 + D_{\mathcal{D}'}\,,
\]
where $D_{\mathcal{D}'} = \E_{Z \sim \mathcal{D}'} \big[\max_{x \in \mathcal{X}} \langle Z, x \rangle - \min_{x \in\mathcal{X}} \langle Z, x \rangle \big]$, and $\mathcal{D}'$ is the probability distribution of the sum of $\ln T$ variables with distribution $\mathcal{D}$.  
\end{lemma}

\begin{proof}
This is an alternative version of Theorem 3.4 of \cite{agarwal2017price}, but with  non-constant learning rates. The proof relies on the following observation, already proved by \cite{agarwal2017price}.
Let  $(x_t)_{t \in [T]}$ denote the predictions of the instance of \FTRL considered in \Cref{thm:mainftrlthm2}.
Consider the alternative experiment in which 
\[
\begin{cases}
Z \textnormal{ is such that } Z \sim \sum_{k \in [\ln T]} Z_k \textnormal{ where } Z_k \sim \mathcal{D}, \\
\hat{x}_1 = x_1,\\
\hat x_{t+1} = \argmin_{X \in \mathcal{X}} ~ \big\langle \sum_{s=1}^t g_{s} + Z, x \big\rangle + \psi_t(x ),
\end{cases}
\]
Then we have
\begin{equation}
\E[\textnormal{Regret}_T(u)]= \E\big[\widehat{\textnormal{Regret}}_T(u)\big]\,,\label{eq:lem_for_agar}
\end{equation}
where $\widehat{\textnormal{Regret}}_T$ is regret of the strategy outputting $\hat x_1, \ldots \hat x_T$.
The complete proof of Equation \eqref{eq:lem_for_agar} can be found in \cite{agarwal2017price}, but we provide its  two main arguments.
First, since the losses are linear, subsequent gradients do not depend on $x_t$. Second, $\sum_{s=1}^t g_{s} + Z$ is distributed as the sanitized sum of gradients used in \FTRL.

Equation \eqref{eq:lem_for_agar} proves that the expected regret is equal to the expectation of $\widehat{\textnormal{Regret}}_T$, the regret of \FTRL with an alternative regularizer $\psi_Z(\cdot) = \psi(\cdot) + \langle Z, \cdot \rangle$. 
The addition of $\langle Z, \cdot \rangle$ does not change the strongly convex nature of the regularizer, since it is linear.
The usual analysis of \FTRL with $\psi_Z$ \citep[Chapter~7]{orabona2019modern} yields
\begin{align*}
\E\big[\widehat{\textnormal{Regret}}_T(u)\big] & \leq  \frac{\psi_Z(u)- \min_{x \in \mathcal{X}} \psi_Z(x)}{\eta_{T}} +\frac{1}{2\mu}\sum_{t = 1}^{T} \eta_{t-1} \|g_t\|_*^2\\
&\leq \frac{\psi(u)- \min_{x \in \mathcal{X}} \psi(x)}{\eta_{T}} +\frac{1}{2\mu}\sum_{t = 1}^{T} \eta_{t-1} \|g_t\|_*^2 + D_{\mathcal{D}'}\,.
\end{align*}
Using this jointly with Equation \eqref{eq:lem_for_agar} concludes the proof.
\end{proof}

%

\begin{lemma}\label{lem:alternative_2}
Let $(p_t)_{t \in[T]}$ be the sequence of probabilities maintained by the instance of \MTFTRL run by agent $j$ among \DOPE, see \Cref{alg:DOPE_learning}.
Consider the alternative experiment in which
\[
\begin{cases}
Z \textnormal{ is such that } Z \sim \sum_{i \in \mathcal{N}_j}\sum_{k \in [\ln T]} Z_k \textnormal{, where } Z_k \sim \mathcal{D}_j,\textnormal{ a Laplacian of dimension } N_j \textnormal{ and parameter } \frac{ \ln T}{\epsilon'}\\
\hat{p}_1 =p_1,\\
\hat{p}_{t+1} = \argmin_{p \in \Delta_j} \left\langle \sum_{s=1}^t \textnormal{loss}_{s} + Z, p \right\rangle + \sum_{\xi \in \Xi_j} p^{(\xi)} \ln p^{(\xi)}
\end{cases}
\]
Then for all $u \in \Delta_j$ we have
\[
\E[\textnormal{Regret}_T(u)]= \E\big[\widehat{\textnormal{Regret}}_T(u)\big]\,,
\]
where $\textnormal{Regret}_T$ and $\widehat{\textnormal{Regret}}_T$ are the regrets of the strategy outputting $p_1, \ldots p_T$ and $\hat p_1, \ldots \hat p_T$ respectively.
\end{lemma}

\begin{proof}
First observe that
the probabilities $p_t$ are the result of running \FTRL where, in place of gradient inputs, a vector whose $\xi$-th element is $\sum_{i \in \mathcal{N}_j} \tilde{s}_{t, i}^{(j, \xi)}$  is used.
Note that the $\tilde{s}_{t, i}^{(j, \xi)}$ come from instances of \treeBasedAgg set with distribution $\mathcal{D}_1$, so that the element $\xi$ of the vector $\sum_{s=1}^t \textnormal{loss}_{s} + Z$ is actually distributed exactly as $\tilde{s}_{t, i}^{(j, \xi)}$.
This eventually shows that $\hat{p_t}$ is distributed as $p_t$.
\end{proof}









We are now ready to prove \Cref{th:reg-DP_anyG}.
\medskip

\regretDP*

\begin{proof}
Let $R^\textnormal{clique-$j$}_T$ be the regret suffered by \MTFTRL run by agent $j$ on the linear losses $\langle w_{i_tj}\,g_t, \cdot\rangle$, with feedback equal to $w_{i_tj}\,\hatGammaOut_{t}^{(i_t)}$ for the sum of gradients incurred by $i_t$, and $\tilde{s}_{t, i_t}^{(j, \xi)}$ for the sums of dot products between experts $(j, \xi)$ and the sum of gradients incurred by $i_t$.
With the same steps used to prove \Cref{lem:cool-cn}, we can prove that 
the regret of \DOPE satisfies
\[
\forall\, U \in \mathbb{R}^{N \times d}, \quad R_T(U) \le \sum_{j=1}^N R^\textnormal{clique-$j$}_T\big(U^{(j)}\big)\,.
\]

By the same arguments as in the proof of \Cref{thm:mt-ftrl} we have
\begin{align*}
R^\textnormal{clique-$j$}_T\big(U^{(j)}\big) 
&\le \underbrace{\sum_{t \colon i_t \in \mathcal{N}_j} \left\langle w_{i_tj}\,g_t,  \big[Y_t^{(j)}\big]_{i_t:} - \big[X_t^{(j,\xi^*)}\big]_{i_t:}\right\rangle}_{\text{Regret of \texttt{Hedge}}} + \underbrace{\sum_{t \colon i_t \in \mathcal{N}_j} \left\langle w_{i_tj}\,g_t,  \big[X_t^{(j,\xi^*)}\big]_{i_t:} - U_{i_t:}\right\rangle}_{\text{Regret with choice $\xi^*$}}\,,
\end{align*}
where $\xi^* = \argmin_{\xi \in\,\Xi_j} \, \sum_{t \colon i_t \in \mathcal{N}_j} \Big\langle w_{i_tj}\,g_t,  \big[X_t^{(j,\xi)}\big]_{i_t:}\Big\rangle$.
We start by upper bounding the regret due to \texttt{Hedge}.
Let $\text{loss}_t \in \mathbb{R}^{N_j}$ storing the $w_{i_tj} \big\langle g_t, \big[X_t^{(j,\xi)}\big]_{i_t:}\big\rangle$ for $\xi \in \Xi_j$, and $e^* \in \mathbb{R}^{N_j}$ the one-hot vector with an entry of $1$ at expert $\xi^*$.
By the analysis of \texttt{Hedge} with regularizers $\psi_t(\bm{p}) = \frac{\beta\sqrt{1+\sum_{s\le t-1}\Ind{i_s \in \scN_j}}}{\sqrt{\ln N_j}} \sum_{k=1}^{N_j} p_k \ln(p_k)$, combined with \Cref{lem:alternative_2},
\begin{align}
\E\left[\sum_{t \colon i_t \in \mathcal{N}_j} \left\langle w_{i_tj}\,g_t,  \big[Y_t^{(j)}\big]_{i_t:} - \big[X_t^{(j,\xi^*)}\big]_{i_t:}\right\rangle \right] &
= \E \left[\sum_{t \colon i_t \in \mathcal{N}_j} \big\langle\,\text{loss}_t,  \bm{p}_t - e^*\big\rangle\nonumber \right]+  D_{\mathcal{D}_{1,j}'}\ \nonumber\\
&= 2 \max_{i \in \scN_j} w_{ij} \sqrt{\ln N_j \sum_{i \in \scN_j} T_i} +  D_{\mathcal{D}_{1,j}'}\,,\label{eq:bound_hedge_2}
\end{align}
where $D_{\mathcal{D}_{1,j}'} = \E_{Z \sim \mathcal{D}_{1,j}'} \big[\max_{x \in \Delta_j} \langle Z, x \rangle - \min_{x \in \Delta_j} \langle Z, x \rangle \big]$ and $\mathcal{D}_{1,j}'$ is the probability distribution of the sum of $N_j \ln T$ variables of distribution $\mathcal{D}_j$, which is a Laplacian of dimension $N_j$ and parameter $\frac{\ln T}{\epsilon'}$.
Hence we have
\begin{align}
D_{\mathcal{D}_{1,j}'} &\leq \E_{Z \sim \mathcal{D}_{1,j}'}  \left[\max_{p \in \Delta_j} \langle Z, p \rangle  - \min_{p \in \Delta_j}  \langle Z, p \rangle \right] \nonumber\\
&\leq 2 \max_{p \in \Delta_j} \|p\|_{1} \E_{Z \sim \mathcal{D}_{1,j}'} \|Z\|_{\infty}\nonumber\\
&\leq 2 \E_{Z \sim \mathcal{D}_{1,j}'} \|Z\|_{\infty}\nonumber\\
&\leq 2 N_j \ln T ~ \E_{Z \sim \mathcal{D}_j} \|Z\|_{\infty} \nonumber\\
&\leq 2 N_j^{3/2} \ln T ~ \E_{Z \sim \mathcal{D}_j} \|Z\|_{2} \label{eq:norm2inf}\\
%
%
&\le 2 N_j^{3/2} \ln T ~ \sqrt{\E_{Z_i \sim \mathcal{D}_1} \left[\sum_{i \in [N_j]} Z_i^2\right]} \nonumber\\
& \le 2 N_j^2 \ln T ~ \sqrt{\E_{Z \sim \mathcal{D}_1} \left[Z^2\right]} \nonumber\\
%
%
&\leq 2 N_j^2 \frac{\ln^2 T}{\epsilon'} \label{eq:variance_lap}\,,
\end{align}
where \eqref{eq:norm2inf} stems form the fact that $Z$ is a vector of dimension $N_j$ and \eqref{eq:variance_lap} stems from the formula of the variance of a Laplacian random variable. 
%
%

Now, we turn to the  regret with choice $\xi^*$.
Assume first that $\sigma_j^2 \le 1$.
Then, the regret with choice $\xi^*$ is in particular better than the regret with choice $\bar{\xi} \in \Xi_j$ such that $\bar{\xi}-\frac{1}{N_j} \le \sigma_j^2 \le \bar{\xi}$.
Recall that the sequence $X_t^{(\bar{\xi})}$ is generated by \FTRL with the sequence of regularizers $\frac{1}{2}\|\cdot\|_{A_j}^2/\eta_{t-1}^{(\bar{\xi})}$.
By the analysis of \FTRL and \Cref{thm:mainftrlthm2} applied to non-adaptive \MTFTRL, we retrieve
\begin{align}
\E\Bigg[\sum_{t \colon i_t \in \mathcal{N}_j} \Big\langle w_{i_tj}\,g_t,  \big[&X_t^{(j,\bar{\xi})}\big]_{i_t:} - U_{i_t:}\Big\rangle \Bigg]
\le 4\,\max_{i \in \scN_j} w_{ij}\, \sqrt{1 + \sigma_j^2(N_j-1)}\sqrt{\sum_{i \in \scN_j} T_i} + D_{\mathcal{D}_{d,j}'}\,,\label{eq:bound_ftrl_3}
\end{align}
where $D_{\mathcal{D}'_{d,j}} = \E_{Z \sim \mathcal{D}'_d} \big[\max_{x \in \mathcal{X}} \langle Z, x \rangle - \min_{x \in\mathcal{X}} \langle Z, x \rangle \big]$ and $\mathcal{D}_{d,j}'$ is the probability distribution of $\mathbb{R}^{N_j\times d}$ matrices whose rows are the sum of $\ln T$ variables of distribution $\mathcal{D}_{d}$, which is a Laplacian of dimension $d$ and parameter $\frac{\sqrt{d} \ln T}{\epsilon'}$. 
Now let us bound the additional term due to the privacy. We have
\begin{align}
D_{\mathcal{D}_{d,j}'} &= \E\bigg[\max_{X \in \mathcal{X}}\langle Z, X \rangle - \min_{X \in \mathcal{X}} \langle Z, X\rangle \bigg] 
\leq 2\E\left[\|Z\|_{A_j^{-1}} ~ \max_{X \in \mathcal{X}}\|X\|_{{A_j}}\right] \nonumber\\
&\le 2 \sqrt N_j  \sqrt{1 + \bar{\xi}(N_j-1)} ~ \E\|Z\|_{A_j^{-1}} \le 2 \sqrt{2 N_j}  \sqrt{1 + \sigma_j^2(N_j-1)} ~ \E\|Z\|_{A_j^{-1}}\,.\label{eq:D_d'}
\end{align}
Now, we have
\begin{align*}
\E\|Z\|_{A_j^{-1}} &\leq \sqrt{\E\|Z\|_{A_j^{-1}}^2}\\
&= \sqrt{\sum_{i,k \in [N_j]}[A_j^{-1}]_{ik} ~ \E\left[Z Z^\top\right]_{ik}}\\
&= \sqrt{\sum_{i \in [N_j]}[A_j^{-1}]_{ii} ~ \E\left[Z_{i:}^\top Z_{i:}\right]}\\
&\leq \sqrt{\sum_{i \in [N_j]}[A_j^{-1}]_{ii}} ~ \sqrt{\sum_{k=1}^d\E\left[Z_{1,k}^2 \right]}\\
&\leq \sqrt{\frac{2N_j}{1+N_j}} \sqrt{d} \sqrt{\textnormal {Var}\left[Z_{1,1}\right]}\\
&\leq \sqrt{2d} \, \sqrt{\ln T \left( \frac{\sqrt{d}  \ln T}{\epsilon'}\right)^2}\\
&\leq  2 d\, \frac{\ln^2 T}{\epsilon'}\,,
\end{align*}
where  the third inequality stems from $[{A_{j}}^{-1}]_{ii}= \frac{2}{N_j+1}$ (see for example computations in Appendix A.2 of \cite{cesa2022multitask}). This in turn implies
\begin{equation}\label{eq:diameter_1}
D_{\mathcal{D}_{d,j}'}\leq 4d\sqrt{2 N_j}  \sqrt{1 + \sigma_j^2(N_j-1)}\,\frac{\ln^2 T}{\epsilon'}\,.
\end{equation}
Assume now that $\sigma_j^2 \ge 1$.
Then, the regret with choice $\xi^*$ is in particular better than the regret with choice $1$.
The latter corresponds to independent learning \citep{cesa2022multitask} and an analysis similar to the one above shows that its regret is bounded by
\begin{align}
\E\Bigg[\sum_{t \colon i_t \in \mathcal{N}_j} \Big\langle w_{i_tj}\,g_t,  \big[X_t^{(j,\bar{\xi})}\big]_{i_t:} - U_{i_t:}\Big\rangle \Bigg] &\leq 
\max_{i \in \scN_j} w_{ij}\,\sqrt{N_j \sum_{i \in \scN_j}T_i} + D_{\mathcal{D}_{d,j}'}\nonumber \\ &
\le \max_{i \in \scN_j} w_{ij} \sqrt{1 + \sigma_j^2(N_j -1)}\sqrt{\sum_{i \in \scN_j}T_i} + D_{\mathcal{D}_{d,j}'}\,.\label{eq:bound_ftrl_4}
\end{align}
In this case, we have
\begin{equation}
D_{\mathcal{D}_{d,j}'}\leq 4d N_j\,\frac{\ln^2 T}{\epsilon'} \leq 4d \sqrt{2 N_j} \sqrt{1 + \sigma_j^2(N_j-1)} \, \frac{\ln^2 T}{\epsilon'} \,.\label{eq:diameter_2}
\end{equation}
Substituting \eqref{eq:variance_lap} into  \eqref{eq:bound_hedge_2} and \eqref{eq:diameter_1} into \eqref{eq:bound_ftrl_3} (or \eqref{eq:diameter_2} into \eqref{eq:bound_ftrl_4}, depending on the value of $\sigma_j^2$), we obtain
\begin{align*}
&\E[R_T(U)]\\
&\qquad\le \sum_{j=1}^N \left[ 6\max_{i \in \scN_j} \wij \left(\sqrt{1 + \sigma_j^2(N_j - 1)} + \ln N_j \right) \sqrt{ \sum_{i \in \scN_j} T_i}  + D_{\mathcal{D}_{1,j}'} + D_{\mathcal{D}_{d,j}'}\right]\\
&\qquad\le \sum_{j=1}^N\left[6 \max_{i \in \scN_j} \wij \left(\sqrt{1 + \sigma_j^2(N_j - 1)} + \ln N_j \right) \sqrt{ \sum_{i \in \scN_j} T_i} + 2 N_j^2 \frac{\ln^2 T}{\epsilon'} + 4d \sqrt{2 N_j} \sqrt{1 + \sigma_j^2(N_j-1)} \, \frac{\ln^2 T}{\epsilon'}\right]\\
&\qquad\tildeO \sum_{j=1}^N \max_{i \in \scN_j} \wij \sqrt{1 + \sigma_j^2(N_j - 1)} \sqrt{ \sum_{i \in \scN_j} T_i}  + \frac{d N^4_\textnormal{max}}{\epsilon}N \ln^2 T\,.
\end{align*}
\end{proof}


\section{Additional Results related to Section \ref{sec:expe}}
\label{sec:app_expe}

In this section, we provide more details about the algorithm used to run \Cref{sec:expe}'s experiments.
In \Cref{sec:expe}, we use Krichevsky-Trofimov's Algorithm (hereafter abbreviated KT) instead of \texttt{Hedge} in order to obtain a $\sigma$-adaptive algorithm.
Practically, this means that \Aclique is set as the algorithm described in \Cref{alg:mt-ftrl-kt}.
In the next theorem, we recall the arguments developed in \citet{cesa2022multitask} and \citet{orabona2019modern} to establish the regret bound of \Cref{alg:mt-ftrl-kt}.

\begin{algorithm}[t]
\caption{~\texttt{MT-FTRL} (with KT adaptive learning rate)}\label{alg:mt-ftrl-kt}
\Req{Learning rates $\beta_t$}
\vspace{0.05cm}

\Init{$A = (1+N) I_N - \bm{1}_N\bm{1}_N^\top$, $z_1 = 0$\\
}\vspace{0.05cm}

\For{$t = 1, 2, \ldots$}{
\vspace{0.15cm}

\tcp{\small \textcolor{blue}{Update the direction using FTRL with Mahalanobis regularizer}}\vspace{0.05cm}

$\displaystyle \widetilde{Y}_t = \argmin_{X : ~ \|X\|_A\leq 1} ~ \frac{\sqrt{N}}{\beta_{t-1}} \left\langle \sum_{s=1}^{t-1} G_s, X\right\rangle + \frac{1}{2} \|X\|_A^2$\vspace{0.15cm}

\tcp{\small \textcolor{blue}{Update the magnitude}}\vspace{0.05cm}

${\displaystyle z_{t} = -\frac{1}{t}\sum_{s=1}^{t-1}u_s \left(1 - \!\!\sum_{s=1}^{t-1} z_s u_s \right) }$\vspace{0.15cm}

\tcp{\small \textcolor{blue}{Predict and get feedback}}\vspace{0.05cm}

Predict with $Y_{t} = z_{t} \widetilde{Y}_{t}$
\vspace{0.1cm}

Pay $\ell_t\big([Y_t]_{i_t:}\big)$ and receives $g_t \in \partial\ell_t\big([Y_t]_{i_t:}\big)$
\vspace{0.05cm}

Set $u_{t} = \frac{\sqrt{N}}{\sqrt{2}L} \big\langle [\tilde{Y}_{t}]_{i_t:}, g_t  \big\rangle$\vspace{0.15cm}

}
\end{algorithm}

\begin{theorem}\label{th:basebound}
Let $G$ be any graph. The regret of \textnormal{\texttt{MT-FTRL}} with KT adaptive learning rate and with $\beta_{t-1}= \max_{i\in \scN_j} \wij \sqrt{1+\sum_{s \le t-1}\Ind{i_s \in \scN_j}}$ satisfies for all $U \in \mathcal{U}$
\[
R_T(U)= \mathcal{O}\left(\sum_{j \in \setN} \sqrt{1 + \sigma_j^2(N_j - 1)}\sqrt{\sum_{i \in \scN_j} T_i} \right)\,,
\]
where we neglect $\log N$ and $\log T$ factors.
\end{theorem}

\begin{proof}
By \Cref{lem:cool-cn}, we have
\[
R_T(U) \le \sum_{j=1}^N R^\textnormal{clique-$j$}_T\big(U^{(j)}\big)\,,
\]
where $R^\textnormal{clique-$j$}_T$ is the regret suffered by \MTFTRL on the linear losses $\langle w_{i_tj}\,g_t, \cdot\rangle$ over the rounds $t \le T$ such that $i_t \in \mathcal{N}_j$.
We proceed by bounding all these terms individually.
Let us focus on $R^\textnormal{clique-$j$}_T$.
To bound this term, we build upon \citet[Theorem~9.9]{orabona2019modern}.
The parameter-independent one-dimensional algorithm $\mathcal{A}_{1d}$ is set as the Krichevsky-Trofimov (KT) algorithm, see \citet[Algorithm~9.2]{orabona2019modern}, and used to learn $ \|U^{(j)}\|_{A_j}$.
It produces the sequence $(z_t^{(j)})_{t \ge 1}$.
Instead, the algorithm $\mathcal{A}_{\mathcal B}$ is set as \texttt{MT-FTRL} on the ball of radius $1$ with respect to norm $\|\cdot\|_{A_{j}}$.
It produces the sequence $(\tilde{Y}_t^{(j)})_{t \ge 1}$.
The prediction at time step $t$ is $Y_t^{(j)} = z_t^{(j)} \tilde{Y}_t^{(j)}$.
We can decompose the regret of \Cref{alg:mt-ftrl-kt} as follows
\begin{align}
R^\textnormal{clique-$j$}_T(U^{(j)}) &\le \sum_{t=1}^T \sum_{i\in \mathcal{N}_j} \big\langle \wij g_t, [Y_t^{(j)}]_{i_t:} - U^{(j)}_{i_t:} \big\rangle \Ind{i_t =i}\nonumber\\
&= \sum_{t: ~i_t \in \scN_j} \big\langle \wij g_t, z_t^{(j)} [\tilde{Y}_t^{(j)}]_{i_t:} - U^{(j)}_{i_t:} \big\rangle\nonumber\\
&= \sum_{t: ~i_t \in \scN_j} \big\langle G^w_t, z_t^{(j)} \tilde{Y}_t^{(j)} - U^{(j)} \big\rangle\nonumber\\
&= \sum_{t: ~i_t \in \scN_j} \big\langle G^w_t, z_t^{(j)} \tilde{Y}_t^{(j)} - \|U^{(j)}\|_{A_{j}} \, \tilde{Y}_t^{(j)} \big\rangle + \sum_{t: ~i_t \in \scN_j} \big\langle G^w_t, \|U^{(j)}\|_{A_{j}} \, \tilde{Y}_t^{(j)} - U^{(j)} \big\rangle\nonumber\\
&= \frac{\max_{i\in \scN_j} \wij\sqrt{2}}{\sqrt{N_j}} ~ \sum_{t: ~i_t \in \scN_j} \frac{\sqrt{N_j}}{\max_{i\in \scN_j} \wij\sqrt{2}} \big\langle G^w_t, \tilde{Y}_t^{(j)}\big\rangle \big(z_t^{(j)} - \|U^{(j)}\|_{A_{j}}\big)\\
&\quad+ \|U^{(j)}\|_{A_{j}} \, \sum_{t: ~i_t \in \scN_j} \left\langle G^w_t, \tilde{Y}_t^{(j)} - \frac{U^{(j)}}{ \|U^{(j)}\|_{A_{j}}} \right\rangle\nonumber\\
&\le \frac{\max_{i\in \scN_j} \wij\sqrt{2}}{\sqrt{N_j}}~\,\text{Regret}_T^{\mathcal{A}_{1d}}\big(\|U^{(j)}\|_{A_{j}}\big)  + \|U^{(j)}\|_{A_{j}}~ \text{Regret}_T^{\mathcal{A}_{\mathcal B}}\left(\frac{U^{(j)}}{ \|U^{(j)}\|_{A_{j}}}\right)\,,\label{eq:decompo_regret}
\end{align}
where $G_t^w$ is the matrix containing only zeros except at row $i_t$, which is equal to $w_{i_tj}g_t$, and $\text{Regret}_T^{\mathcal{A}_{1d}}$ and $\text{Regret}_T^{\mathcal{A}_{\mathcal B}}$ denote upper bounds on the regrets (with linear losses) of algorithms $\mathcal{A}_{1d}$ and $\mathcal{A}_\mathcal{B}$ respectively.
Note that the last inequality holds as we do have
\[
\left|\frac{\sqrt{N_j}}{\max_{i\in \scN_j} \wij\sqrt{2}} \big\langle G^w_t, \tilde{Y}_t^{(j)} \big\rangle\right| \leq \frac{\sqrt{N_j}}{\max_{i\in \scN_j} \wij\sqrt{2}} \, \|G^w_t\|_{A_{j}^{-1}}\,\|\tilde{Y}_t^{(j)}\|_{A_{j}} = \frac{w_{i_tj}\,\sqrt{N_j}}{\max_{i\in \scN_j} \wij\sqrt{2}}\,\sqrt{[A_{j}^{-1}]_{i_ti_t}} \, \|g_t\| \le 1\,.
\]
Now, using \citet[Section~9.2.1]{orabona2019modern}, we have that there exists a universal constant $C_0$ such that
\begin{align}
\text{Regret}_T^{\mathcal{A}_{1d}}\big(\|U^{(j)}\|_{ {A_{j}}}\big) &\leq \| U^{(j)}\|_{{A_{j}}} \sqrt{4 \sum_{i \in \scN_j} T_i \ln\Big(1 + C_0 \|U^{(j)}\|_{ {A_{j}}} \sum_{i \in \scN_j} T_i\Big)} + 1\nonumber\\
&\leq \sqrt{ N_j\big(1+\sigma^2(N_j-1)\big)} \sqrt{4 \Big(\sum_{i \in \scN_j} T_i\Big) \ln\Big(1 +  C_0  N_j \sum_{i \in \scN_j} T_i\Big)} + 1\,,\label{eq:regret_1d}
\end{align}
where we used \citet[Equation~(21)]{cesa2022multitask} to compute $\|U^{(j)}\|_{ {A_{j}}}$.
%
%
%
On the other hand, \texttt{MT-FTRL} is an instance of \texttt{FTRL} with regularizer $X \mapsto (1/2)\|X\|^2_{A_{j}}$ and learning rates $\eta_{t-1} = 1/ \beta_{t-1}$.
Its regret on the unit ball with respect to $\|\cdot\|_{{A_{j}}}$ can thus be bounded (see e.g., \citealt[Corollary~7.9]{orabona2019modern}) by
\begin{align}
\text{Regret}_T^{\mathcal{A}_{\mathcal B}}\left(\frac{U^{(j)}}{\| U^{(j)}\|_{{A_{j}}}}\right) & \leq \frac{\left\|\frac{U^{(j)}}{\| U^{(j)}\|_{ {A_{j}}}}\right\|_{{A_{j}}}}{2\,\eta_{T-1}} + \frac{1}{2}\sum_{t: ~i_t \in \scN_j} \eta_{t-1} \big\|G^w_t\big\|_{A_{j}^{-1}}^2\nonumber\\
&\leq \frac{1}{2\eta_{T-1}} + \frac{1}{2} \sum_{t: ~i_t \in \scN_j} \eta_{t-1} [A_{j}^{-1}]_{i_ti_t} w_{i_tj}^2 \|g_t\|_2^2\nonumber\\
&\leq \frac{\max_{i\in \scN_j} \wij\,\sqrt N_j}{2} \sqrt{\sum_{i \in \scN_j} T_i}  + \frac{1}{\max_{i\in \scN_j} \wij} \sum_{t: ~i_t \in \scN_j} \frac{w_{i_tj}^2}{\sqrt {1+\sum_{s \le t-1} \Ind{i_s \in \scN_j}}}\label{eq:for_stoch_KT}\\
&\leq \sqrt N_j \max_{i\in \scN_j} \wij\frac{\sqrt{\sum_{i \in \scN_j} T_i}}{2} + \frac{2\,\max_{i\in \scN_j} \wij \, \sqrt N_j \sqrt{\sum_{i \in \scN_j} T_i}}{N_j}  \nonumber\\
&= 2\max_{i\in \scN_j} \wij \sqrt{\frac{\sum_{i \in \scN_j} T_i}{N_j}}\,,\label{eq:regret_ball}
\end{align}
where the third inequality comes from $[{A_{j}}^{-1}]_{ii}= \frac{2}{N_j+1}$ (see for example computations in Appendix A.2 of \cite{cesa2022multitask}).
Substituting \eqref{eq:regret_1d} and \eqref{eq:regret_ball} into \eqref{eq:decompo_regret}, we obtain
\begin{align*}
&R_T(U^{(j)})\\
&~~\le \max_{i\in \scN_j} \wij \sqrt{1+\sigma_j^2(N_j-1)} \sqrt{\sum_{i \in \scN_j} T_i}\left(2 + \sqrt{8\ln\Big(1 +  C_0 N_j {\sum_{i \in \scN_j} T_i}\Big)}\right) + \max_{i\in \scN_j} \wij \sqrt{1+\sigma_j^2(N_j-1)} \sqrt{\sum_{i \in \scN_j} T_i}\\
&~~\le \max_{i\in \scN_j} \wij \sqrt{1+\sigma_j^2(N_j-1)} \sqrt{\sum_{i \in \scN_j} T_i}\left(3 + \sqrt{8\ln\Big(1 +  C_0  N_j \sum_{i \in \scN_j} T_i\Big)}\right)\,.
\end{align*}
\end{proof}
\medskip

\begin{remark}
Note that in the stochastic case, \Cref{eq:for_stoch_KT} shows that we could use $\beta_{t-1} = \sqrt{\sum_{i \in \scN_j}\frac{q_i}{Q_j}\,w_{ij}^2}$ $\sqrt{1+\sum_{s \le t-1}\Ind{i_s \in \scN_j}}$   for each agent $j$as in \Cref{thm:sto}, therefore reducing the expectation of the second term of the regret of \Cref{eq:decompo_regret} to a $\tilde{\mathcal{O}}\left(\sum_{j=1}^N \sqrt{\sum_{i\in \mathcal{N}_j}q_i\,w_{ij}^2} \sqrt{1 + \sigma_j^2 (N_j-1)}\right) \hspace{-0.1cm}\sqrt{T}$. However, that would not change the first term in the regret of \Cref{eq:decompo_regret}, related to the regret of the KT-algorithm itself.
\end{remark}

\end{document}